\documentclass[twocolumn,fleqn]{autart}    

\usepackage{epstopdf}
\usepackage{graphicx}          
\usepackage[dvips]{epsfig}    
\usepackage{amsmath,amssymb,bm}
\newtheorem{theorem}{Theorem}
\newtheorem{lemma}{Lemma}

\newtheorem{assumption}{Assumption}
\newtheorem{definition}{Definition}
\newenvironment{proof}[1][Proof]{\par\noindent\textbf{#1.} }{\qed\par}

\usepackage{booktabs}
\usepackage{multirow}
\usepackage{array}
\newcolumntype{M}[1]{>{\centering\arraybackslash}m{#1}}

\usepackage{graphicx}
\allowdisplaybreaks[1]

\usepackage[round,colon,longnamesfirst]{natbib} 

\usepackage{hyperref}
\usepackage{xcolor}
\hypersetup{
    colorlinks=true, 
    linkcolor=blue,  
    urlcolor=blue,   
    citecolor=blue   
}

\begin{document}
\setlength{\mathindent}{0pt} 
\begin{frontmatter}

\title{A Weighted Gradient Tracking Privacy-Preserving Method for Distributed Optimization\thanksref{footnoteinfo}} 

\thanks[footnoteinfo]{This work is supported by National Natural Science Foundation of China under grant number 62173259. The material in this paper was not presented at any conference. Corresponding author Bing Liu.}

\author[Wuhan]{Furan Xie}\ead{xiefuran328@wust.edu.cn},    
\author[Wuhan]{Bing Liu}\ead{liubing17@wust.edu.cn},               
\author[Hangzhou]{Li Chai}\ead{chaili@zju.edu.cn}  

\address[Wuhan]{Engineering Research Center of Metallurgical Automation and Measurement Technology, Wuhan University of Science and Technology, Wuhan 430081, China}  
\address[Hangzhou]{College of Control Science and Engineering, Zhejiang University, Hangzhou 310027, China}        


\begin{abstract}                          
This paper investigates the privacy-preserving distributed optimization problem, aiming to protect agents' private information from potential attackers during the optimization process. Gradient tracking, an advanced technique for improving the convergence rate in distributed optimization, has been applied to most first-order algorithms in recent years. We first reveal the inherent privacy leakage risk associated with gradient tracking. Building upon this insight, we propose a weighted gradient tracking distributed privacy-preserving algorithm, eliminating the privacy leakage risk in gradient tracking using decaying weight factors. Then, we characterize the convergence of the proposed algorithm under time-varying heterogeneous step sizes. We prove the proposed algorithm converges precisely to the optimal solution under mild assumptions. Finally, numerical simulations validate the algorithm's effectiveness through a classical distributed estimation problem and the distributed training of a convolutional neural network.
\end{abstract}

\end{frontmatter}

\section{Introduction}
With the emergence of networked systems, distributed optimization has attracted widespread attention and has been widely applied in multi-agent systems \citep{yi2023convergence,wang2024differentially}, distributed learning \citep{liu2024convergence}, and power systems \citep{liu2024scalable,cai2024distributed}. However, due to the large number of sensitive information involved in practical applications, the privacy leakage risk in distributed optimization has become a major obstacle to its further development and application. Developing privacy-preserving distributed optimization algorithms, without compromising other performance such as accuracy or robustness, is one of the key challenges of current research. Existing privacy-preserving algorithms can broadly be categorized into three types: differential privacy-based methods, encryption-based methods, and correlated randomness-based methods \citep{wang2024privacy}, which will be briefly reviewed in the following three paragraphs.

Differential privacy-based methods achieve privacy protection by injecting random noise into the interaction information between agents. \citet{huang2015differentially} and \citet{ding2022differentially} respectively propose differential private distributed optimization algorithms based on the distributed projected gradient descent algorithm and the gradient tracking algorithm. However, these two algorithms fail to converge to the optimal solution due to the loss of optimization accuracy caused by the injected random noise. In addition, the gradient tracking technique leads to the accumulation of noises, which is more detrimental to accurate convergence. Inspired by the robustness idea in \cite{pu2020robust}, \citet{yu2023gradient} presents a robust gradient-tracking based differential private distributed optimization algorithm, which avoids noise accumulation and improves the optimization accuracy. Nevertheless, \cite{wang2024tailoring} has shown that directly integrating the differential privacy mechanism into existing methods inevitably compromises optimization accuracy. To this end, a decaying factor is introduced to suppress the effect of noise, which can guarantee differential privacy and almost sure convergence to the optimal solution. However, this approach has a slower convergence rate. Due to the importance of optimization accuracy and convergence rate to the algorithm, it is of great importance to find a better trade-off between the level of privacy and these performances.


Encryption-based methods protect privacy by using cryptographic techniques to encrypt the interaction information between agents. The Paillier cryptosystem is a widely used encryption technique. Based on a weight decomposition strategy, \cite{ruan2019secure} first applies the Paillier cryptosystem to average consensus in a fully distributed manner. Inspired by \cite{ruan2019secure}, \cite{zhang2019admm} and \cite{zhang2019enabling} respectively develop encryption-based privacy-preserving algorithms by combining the Paillier cryptosystem with the distributed ADMM and the distributed projected subgradient descent method. However, these methods are limited to undirected graphs and have heavy computational and communication burdens. By using the advanced encryption standard (AES), a lightweight privacy-preserving algorithm is proposed in \cite{liu2024cryptographybased}, which not only significantly reduces the computational complexity, but also is applicable to time-varying directed graphs. Despite these improvements, encryption-based methods still need extra computation and communication costs, which will face challenges in dealing with large-scale optimization problems.

Correlated randomness-based methods achieve privacy protection by introducing spatially or temporally correlated randomness to the interaction information between agents. \cite{wang2019privacypreserving} proposes a state decomposition strategy, which protects the agents' initial information in average consensus by decomposing their states into two random sub-states. Following \cite{wang2019privacypreserving}, a state decomposition-based privacy-preserving algorithm for distributed optimization is presented in \cite{cheng2024privacypreserving}. In \cite{gade2018private}, the private information of agents is protected by injecting structured random noise into the interaction information between agents. \cite{huan2023dynamics} utilizes the robustness of system dynamics and enables privacy protection by increasing the randomness of optimization parameters. \cite{li2020privacypreserving} protects the agents' private information by injecting random noise into the non-convergent subspace via the dual variable. However, the above correlated randomness-based privacy-preserving methods require an additional topological assumption: each agent must have at least one legitimate neighbor who does not participate in attacks. From an implementation point of view, it is challenging to ensure this assumption holds. In addition, these methods are vulnerable to attackers who can access all messages shared within the communication network.

To overcome the aforementioned limitations, in this paper, we first reveal how advanced distributed optimization methods disclose the private information of agents, and then propose a novel privacy-preserving algorithm. In summary, the main contributions of this paper include: 1) We propose a weighted gradient tracking-based privacy-preserving algorithm for distributed optimization over directed graphs, which circumvents the inherent privacy leakage risk in gradient tracking by introducing decaying weight factors. 2) The proposed algorithm protects agents' private information without compromising optimization accuracy, increasing computational or communication burdens, or imposing additional topological assumptions. This is quite different from differential privacy-based methods that trade optimization accuracy for privacy, from encryption-based methods that introduce computational and communication overhead, and from correlated randomness-based methods that impose additional topological assumptions. 3) We rigorously prove the convergence of the proposed algorithm under the general assumption that the objective function is strongly convex and smooth. The incorporation of decaying weight factors reduces the algorithm's exploration of the state space, thereby introducing analytical challenges and making the analysis methods in the existing literature inapplicable. Through our analysis, we deduce sufficient conditions for the decaying weight factors that ensure the algorithm converges to the optimal solution, and we characterize the convergence behavior of the algorithm under time-varying heterogeneous parameters.


The organization of the paper is as follows. Section \ref{section_problem} introduces the problem formulation, reveals the privacy leakage risk in gradient tracking, and proposes a weighted gradient tracking privacy-preserving algorithm. Section \ref{section_convergence} and \ref{section_privacy} prove the convergence and privacy respectively. Section \ref{section_simulation} validates the effectiveness. Finally, Section \ref{section_conclusion} concludes the paper.

\textit{Notation:}
Consider a directed graph $\mathbb{G}=([n],\mathcal{E})$, where $[n]=\{1, 2, \cdots, n\}$ is the set of agents and $\mathcal{E} \subseteq [n]\times[n]$ is the edge set of ordered pairs of agents. The sets $\mathcal{N}_{i}^{\text{out}} = \{j|(i,j)\in\mathcal{E}\}$ and $\mathcal{N}_{i}^{\text{in}} = \{j|(j,i)\in\mathcal{E}\}$ represent the out-neighbors and the in-neighbors of agent $i$. We use $\mathbf{1}$ and $\mathbf{0}$ to represent all-ones and all-zeros column vectors, and $I$ to denote the identity matrix. Their dimensions are determined by the context. For any vector $v$, $\| v \|_{2}$ denotes the Euclidean norm of $v$, and $v_{i}$ represents the $i$-th entry of $v$. If $v$ is time-varying, the $i$-th entry at time $k$ is denoted by $[v_{k}]_{i}$. For any two vectors $v$ and $u$ with appropriate dimensions, $\langle v,u \rangle$ denotes the inner product of $v$ and $u$. All vectors are considered to be column vectors unless explicitly stated otherwise. For any matrix $M$, $\| M \|_{2}$ denotes the spectral norm of $M$, $M^{*}$ denotes the conjugate transpose of $M$, and $M_{ij}$ represents the $ij$-th entry of $M$. If $M$ is time-varying, the $ij$-th entry at time instant $k$ is denoted by $[M_{k}]_{ij}$. For a square matrix $M$, $\rho(M)$ and $\det(M)$ denote the spectral radius and the determinant respectively. Given a sequence of matrices $M_{k}$, we define $\prod_{k=K_{1}}^{K_{2}}M_{k} = M_{K_{2}}\cdots M_{K_{1}}$ if $K_{2}\ge K_{1}$ and $\prod_{k=K_{1}}^{K_{2}}M_{k} = I$ if $K_{2}< K_{1}$.




\section{Problem Formulation and Algorithm}\label{section_problem}
\subsection{Problem Formulation}
We consider a system consisting of $n$ agents over a directed graph $\mathbb{G}$. The goal of these agents is to collaboratively solve the following global objective function via local computation and local communication:
\begin{equation}\label{problem_optimization}
\min_{x\in\mathbb{R}^{p}} f(x) = \sum_{i=1}^{n} f_{i}(x),
\end{equation}
where $x \in \mathbb{R}^{p}$ is the decision variable among all agents and $f_{i}: \mathbb{R}^{p} \to \mathbb{R}$ is the local objective function of agent $i$.

For subsequent analysis, we make the following standard assumptions about the objective function and the communication graph:
\begin{assumption}\label{assumption_smooth}
Each local objective function $f_{i}$ is continuously differentiable and has $L$-Lipschitz continuous gradients, i.e., there exists a Lipschitz constant $L > 0$ such that
\begin{equation}
\| \nabla f_{i}(x)-\nabla f_{i}(u) \|_{2} \le L\| x-u \|_{2},
\end{equation}
for all $i\in [n]$ and $x,u\in\mathbb{R}^{p}$.
\end{assumption}

\begin{assumption}\label{assumption_convex}
Each local objective function $f_{i}$ is $\mu$-strongly convex, i.e.,
\begin{equation}
\left\langle \nabla f_{i}(x)-\nabla f_{i}(u),x-u \right\rangle \ge \mu\| x-u \|_{2}^{2},
\end{equation}
for all $i\in [n]$ and $x,u\in\mathbb{R}^{p}$, where $\mu > 0$ is a constant.
\end{assumption}

\begin{assumption}\label{assumption_connected}
The directed graph $\mathbb{G} = ([n],\mathcal{E})$ is strongly connected, i.e., there is a directed path between any two distinct agents.
\end{assumption}

Under Assumptions \ref{assumption_smooth} and \ref{assumption_convex}, $f(x)$ is $\hat{L}$-smooth and $\hat{\mu}$-strongly convex, where $\hat{L} = nL$ and $\hat{\mu} = n\mu$. This implies that the problem (\ref{problem_optimization}) has a unique optimal solution $x^{*}$.

For the problem (\ref{problem_optimization}), we define the agents' gradients and intermediate states (except the optimal state) as their private information. This is due to the fact that revealing the gradients and intermediate states may disclose agents' sensitive information. For example, in the distributed rendezvous problem, disclosing an agent's intermediate states will directly expose its initial position \citep{mo2016privacy}. Similarly, in the distributed localization problem, adversaries can use the agents' gradients and states to infer their locations \citep{alanwar2017proloc}. In adversarial environments, the positions or locations of agents are often considered sensitive information. Moreover, in distributed machine learning scenarios, adversaries can use an agent's gradient to recover its raw training data \citep{zhu2019deep,geiping2020inverting}. This recovery is pixel-wise accurate for images and token-wise matching for texts. Notably, the raw training data typically contains sensitive information such as medical records and salary information.


We consider two common types of attackers in the field of privacy preservation: the \textit{external eavesdropper} and the \textit{honest-but-curious adversary} \citep{goldreich2009foundations}. An external eavesdropper has knowledge of the topology of the communication network and can eavesdrop on all communication channels and intercept the exchanged messages to infer the agents' private information. An honest-but-curious adversary is a participating agent in the system that follows all the protocols correctly. However, it is curious and collects the received messages to infer other agents' private information. Except for the honest-but-curious agents, the remaining agents in the system are referred to as the \textit{honest agents}, who also follow all the protocols correctly but have no intention to infer other agents' private information. Note that the information accessible to the two types of attackers is distinct. An external eavesdropper can intercept all messages exchanged between agents, whereas an honest-but-curious adversary is limited to receiving messages only from its in-neighbor agents. Furthermore, an honest-but-curious adversary can access internal system information, such as the optimization algorithm's update protocol, which remains inaccessible to the external eavesdropper. In this paper, we assume that the honest-but-curious adversaries can collude among themselves and with external eavesdroppers to infer the honest agents' private information. In such scenarios, the risk of privacy leakage becomes more severe.

\subsection{Privacy Leakage in Gradient Tracking}\label{section_privacy_leakage}
In existing first-order optimization methods for solving the problem (\ref{problem_optimization}), the gradient tracking (GT)-based method stands out as one of the most attractive techniques. GT-based methods introduce an auxiliary variable to estimate the average global gradient and use this estimated gradient as its descent direction, achieving a linear convergence rate under constant step sizes \citep{shi2015extra,nedic2017achieving,xin2019frost,saadatniaki2020decentralized,pu2020push}.  However, GT-based methods cannot protect the private information of agents, especially the agents' gradients. In the following, we will take the AB algorithm \citep{saadatniaki2020decentralized} as an example to reveal the privacy leakage issue in GT-based methods.

In the AB algorithm, each agent $i \in [n]$ maintains two variables $x_{i}$ and $y_{i}$ at each iteration $k$, where $x_{i}\in\mathbb{R}^{p}$ is a local copy of the decision variable and $y_{i}\in\mathbb{R}^{p}$ is an estimate of the average global gradient. In addition, the AB algorithm applies to a directed graph $\mathbb{G}$ and uses both row-stochastic and column-stochastic weight matrices. Denote $A_{k}$ and $B_{k}$ as time-varying row-stochastic and column-stochastic weight matrices associated with the directed graph $\mathbb{G}$. At each iteration $k$, agent $i$ sends $x_{i}^{k}$ and $[B_{k}]_{li}y_{i}^{k}$ to its out-neighbors $l\in\mathcal{N}_{i}^{\text{out}}$ and receives $x_{j}^{k}$ and $[B_{k}]_{ij}y_{j}^{k}$ from its in-neighbors $j\in\mathcal{N}_{i}^{\text{in}}$. Let $\alpha$ be a constant step size. Then, the update of agent $i$ in the AB algorithm is as follows: for all $k\ge1$,
\begin{subequations}
\begin{align}
&x_{i}^{k+1} = \sum_{j\in\mathcal{N}_{i}^{\text{in}}\cup\{i\}}[A_{k}]_{ij}x_{j}^{k} - \alpha y_{i}^{k}, \\
&y_{i}^{k+1} = \sum_{j\in\mathcal{N}_{i}^{\text{in}}\cup\{i\}}[B_{k}]_{ij}y_{j}^{k} + \nabla f_{i}(x_{i}^{k+1}) - \nabla f_{i}(x_{i}^{k}),
\end{align}
\end{subequations}
where $\nabla f_i$ is the gradient of $f_{i}$ and $y_{i}^{1} = \nabla f_{i}(x_{i}^{1})$.

Since $[B_{k}]_{ii} + \sum_{l\in\mathcal{N}_{i}^{\text{out}}}[B_{k}]_{li} = 1$, the update of $y_{i}$ can be rewritten as
\begin{equation}\label{equation_rewrite_y}
\begin{aligned}
y_{i}^{k+1} - y_{i}^{k} = &\ \nabla f_{i}(x_{i}^{k+1}) - \nabla f_{i}(x_{i}^{k}) \\
&+ \sum_{j\in\mathcal{N}_{i}^{\text{in}}}[B_{k}]_{ij}y_{j}^{k} - \sum_{l\in\mathcal{N}_{i}^{\text{out}}}[B_{k}]_{li}y_{i}^{k}.
\end{aligned}
\end{equation}
In (\ref{equation_rewrite_y}), $[B_{k}]_{ij}y_{j}^{k}$ is sent to agent $i$ by its in-neighbors $j\in\mathcal{N}_{i}^{\text{in}}$, while $[B_{k}]_{li}y_{i}^{k}$ is sent by agent $i$ to its out-neighbors $l\in\mathcal{N}_{i}^{\text{out}}$. Let $z_{i}^{k} = \sum_{l\in\mathcal{N}_{i}^{\text{out}}}[B_{k}]_{li}y_{i}^{k} - \sum_{j\in\mathcal{N}_{i}^{\text{in}}}[B_{k}]_{ij}y_{j}^{k}$. By recursion, we have
\begin{equation}
\begin{aligned}
y_{i}^{k+1} - y_{i}^{1} = \nabla f_{i}(x_{i}^{k+1}) - \nabla f_{i}(x_{i}^{1}) - \sum_{m=1}^{k}z_{i}^{m}.
\end{aligned}
\end{equation}
Since $y_{i}^{1} = \nabla f_{i}(x_{i}^{1})$, it follows that
\begin{equation}\label{equation_leakage_GT}
\begin{aligned}
y_{i}^{k+1} = \nabla f_{i}(x_{i}^{k+1}) - \sum_{m=1}^{k}z_{i}^{m}.
\end{aligned}
\end{equation}
As $k\to\infty$, $y_{i}^{k+1}\to 0$ and $x_{i}^{k+1} \to x^*$. Then, $\nabla f_{i}(x^{*}) = \sum_{m=1}^{\infty}z_{i}^{m}$. Thus, once $z_{i}^{k},\forall k$ are leaked, the attacker can accurately infer the gradient $\nabla f_{i}(x^{*})$ of agent $i$.

Given that external eavesdroppers have access to messages on all communication channels, they can consistently obtain $z_{i}^{k},\forall k$. Furthermore, if all neighbors of agent $i$ are honest-but-curious adversaries, these adversaries can also obtain $z_{i}^{k},\forall k$. Therefore, the AB algorithm cannot protect the gradient $\nabla f_{i}(x^{*})$ of agent $i$. In fact, this privacy leakage is unavoidable for almost all GT-based methods. In addition, the AB algorithm cannot protect the intermediate states of agents due to the exchange of actual states $x_{i}^{k},\forall k$ among agents.

\subsection{A Weighted Gradient Tracking Privacy-Preserving Method}\label{subsection_proposed_algorithm}
To address the privacy leakage issue in GT-based methods, we propose a privacy-preserving distributed algorithm based on weighted gradient tracking, as depicted below. First, to protect the intermediate states of agents, we employ an adapt-then-combine update for $x_{i}$ and a heterogeneous step size $\alpha_{i}$. Secondly, to avoid gradient leakage as demonstrated in Section \ref{section_privacy_leakage}, we introduce a decaying weight factor $\lambda_{k}$ in the update of $y_{i}$, satisfying $\lambda_{k} \to 0$ as $k\to\infty$. 
\begin{subequations}\label{our_algorithm}
\begin{align}
&x_{i}^{k+1} = \sum_{j\in\mathcal{N}_{i}^{\text{in}}\cup\{i\}}[A_{k}]_{ij}(x_{j}^{k} - \alpha_{j} y_{j}^{k}), \\
&y_{i}^{k+1} = \sum_{j\in\mathcal{N}_{i}^{\text{in}}\cup\{i\}}[B_{k}]_{ij}y_{j}^{k} + \lambda_{k+1} \nabla f_{i}(x_{i}^{k+1}) - \lambda_{k} \nabla f_{i}(x_{i}^{k}),
\end{align}
\end{subequations}
where $x_{i}^{1} \in \mathbb{R}^{p}$ and $y_{i}^{1} = \lambda_{1}\nabla f_{i}(x_{i}^{1})$. Note that in (\ref{our_algorithm}), the $y_{i}^{k}$ no longer tracks the average global gradient $\frac{1}{n}\sum_{i=1}^{n}\nabla f_{i}(x_{i}^{k})$ but the weighted average global gradient $\frac{1}{n}\lambda_{k}\sum_{i=1}^{n}\nabla f_{i}(x_{i}^{k})$, which is crucial for achieving privacy protection.

For time-varying weight matrices $A_{k}$ and $B_{k}$, we make the following assumptions:
\begin{assumption}\label{assumption_A}
For each iteration $k$, the matrix $A_{k}$ is row-stochastic and compatible with the directed graph $\mathbb{G}$, and there exists a positive number $a$ such that the minimum positive entry of $A_{k}$ is not less than $a$, i.e.,
\begin{equation}
[A_{k}]_{ij} \left\{
\begin{aligned}
\ge a,\quad \forall j\in\mathcal{N}_{i}^{\text{in}}\cup \{i\} \\
= 0, \quad \forall j\notin\mathcal{N}_{i}^{\text{in}}\cup \{i\}
\end{aligned}
\right. \textit{and}\ A_{k}\mathbf{1} = \mathbf{1}.
\end{equation}
\end{assumption}
\begin{assumption}\label{assumption_B}
For each iteration $k$, the matrix $B_{k}$ is column-stochastic and compatible with the directed graph $\mathbb{G}$, and there exists a positive number $b$ such that the minimum positive entry of $B_{k}$ is not less than $b$, i.e.,
\begin{equation}
[B_{k}]_{li} \left\{
\begin{aligned}
\ge b,\quad \forall l\in\mathcal{N}_{i}^{\text{out}}\cup \{i\} \\
= 0, \quad \forall l\notin\mathcal{N}_{i}^{\text{out}}\cup \{i\}
\end{aligned}
\right. \textit{and}\ \mathbf{1}^{T}B_{k} = \mathbf{1}^{T}.
\end{equation}
\end{assumption}

In the following, we will explain how our algorithm avoids gradient leakage. The detailed privacy analysis is left to Section \ref{section_privacy}. Similar to the derivation of (\ref{equation_leakage_GT}), we have the following equality under the update equation (\ref{our_algorithm}):
\begin{equation}\label{equation_leakage_WGT}
\begin{aligned}
y_{i}^{k+1} = \lambda_{k+1}\nabla f_{i}(x_{i}^{k+1}) - \sum_{m=1}^{k}z_{i}^{m}.
\end{aligned}
\end{equation}
As $k\to\infty$, $\lambda_{k+1}\to 0$. Assuming the algorithm converges, then $y_{i}^{k+1}\to 0$ and $x_{i}^{k+1} \to x^*$. Even if the adversaries obtain $z_{i}^{k}, \forall k$, they cannot infer the gradient $\nabla f_{i}(x^{*})$ because the decaying weight factor $\lambda_{k+1}\to 0$ erases the gradient $\nabla f_{i}(x^{*})$. That is to say, our algorithm mitigates the privacy breach risk of gradient leakage in GT-based methods by introducing the decaying weight factor $\lambda_{k}$.

\section{Convergence Analysis}\label{section_convergence}
In this section, we first provide some basic definitions and lemmas. Then, we give the algorithm's convergence result. The update equation (\ref{our_algorithm}) can be written in a compact form as
\begin{subequations}
\begin{align}
&\mathbf{x}^{k+1} = A_{k}(\mathbf{x}^{k} - \bm{\alpha}\mathbf{y}^{k}), \label{our_compact_algorithm_x}\\
&\mathbf{y}^{k+1} = B_{k}\mathbf{y}^{k} + \lambda_{k+1}\nabla F(\mathbf{x}^{k+1}) - \lambda_{k}\nabla F(\mathbf{x}^{k}),\label{our_compact_algorithm_y}
\end{align}
\end{subequations}
where $\bm{\alpha} := \text{diag}\{\alpha_{1}, \alpha_{2}, \cdots, \alpha_{n}\} \in \mathbb{R}^{n\times n}$ is a diagonal matrix and $\mathbf{x}^{k}$, $\mathbf{y}^{k}, \nabla F(\mathbf{x}^{k})$ are given by 
\begin{align*}
&\mathbf{x}^{k} := \big(x_{1}^{k}, x_{2}^{k}, \cdots, x_{n}^{k}\big)^{T} \in \mathbb{R}^{n\times p}, \\
&\mathbf{y}^{k} := \big(y_{1}^{k}, y_{2}^{k}, \cdots, y_{n}^{k}\big)^{T} \in \mathbb{R}^{n\times p}, \\
&\nabla F(\mathbf{x}^{k}) := \big(\nabla f_{1}(x_{1}^{k}), \nabla f_{2}(x_{2}^{k}), \cdots, \nabla f_{n}(x_{n}^{k})\big)^{T} \in \mathbb{R}^{n\times p}.
\end{align*}

From (\ref{our_compact_algorithm_y}), we deduce $\mathbf{1}^{T}\mathbf{y}^{k+1} - \lambda_{k+1}\mathbf{1}^{T}\nabla F(\mathbf{x}^{k+1}) = \mathbf{1}^{T}\mathbf{y}^{k} - \lambda_{k}\mathbf{1}^{T}\nabla F(\mathbf{x}^{k}) = \cdots = \mathbf{1}^{T}\mathbf{y}^{1} - \lambda_{1}\mathbf{1}^{T}\nabla F(\mathbf{x}^{1})$ due to $\mathbf{1}^{T}B_{k} = \mathbf{1}^{T}$. Because of $\mathbf{1}^{T}\mathbf{y}^{1} = \lambda_{1}\mathbf{1}^{T}\nabla F(\mathbf{x}^{1})$, we have 
\begin{equation*}
\mathbf{1}^{T}\mathbf{y}^{k} = \lambda_{k}\mathbf{1}^{T}\nabla F(\mathbf{x}^{k}),\ \textit{for all}\ k\ge1.
\end{equation*}

In addition, we define two notations for subsequent analysis: $\bar{x}^{k} := \phi_{k}^{T}\mathbf{x}^{k} \in \mathbb{R}^{1\times p}$ and $\hat{y}^{k} := \mathbf{1}^{T}\mathbf{y}^{k} \in \mathbb{R}^{1\times p}$, i.e., the weighted average of states and the sum of directions, where $\phi_{k}$ is a stochastic vector which will be defined in Lemma \ref{lemma_phi}.

\subsection{Basic Definitions and Lemmas}
Inspired by \cite{pu2020push}, we give the definition of the matrix norm for any $\mathbf{x}\in\mathbb{R}^{n\times p}$.
\begin{definition}\label{definition_x}
Given an arbitrary vector norm $\|\cdot\|$ in $\mathbb{R}^{n}$, we define the matrix norm for any $\mathbf{x}\in\mathbb{R}^{n\times p}$ as
\begin{equation}
\|\mathbf{x}\| = \Big\|\Big(\|\mathbf{x}^{(1)}\|, \|\mathbf{x}^{(2)}\|, \cdots, \|\mathbf{x}^{(p)}\|\Big)\Big\|_{2},
\end{equation}
where $\mathbf{x}^{(1)}, \mathbf{x}^{(2)}, \cdots, \mathbf{x}^{(p)} \in \mathbb{R}^{n}$ are columns of $\mathbf{x}$.
\end{definition}

For the stochastic matrices $A_{k}$ and $B_{k}$ with Assumptions \ref{assumption_connected}, \ref{assumption_A}, and \ref{assumption_B}, we elaborate the implications of their stochastic nature in Lemmas \ref{lemma_phi}, \ref{lemma_pi}, and \ref{lemma_sigma_AB}.
\begin{lemma}\cite[Lemma 3.3]{nedic2023ab}\label{lemma_phi}
There exists a sequence $\{\phi_{k}\}$ of stochastic vectors such that $\phi_{k+1}^{T}A_{k} = \phi_{k}^{T}$ for all $k\ge1$, where each $\phi_{k}$ satisfies $[\phi_{k}]_{i} \ge a^{n}/n$ for all $i\in[n]$ and $a$ is the lower bound of positive entries of $A_{k}$.
\end{lemma}

\begin{lemma}\cite[Lemma 3.4]{nedic2023ab}\label{lemma_pi}
Define the sequence $\{\pi_{k}\}$ of stochastic vectors as $\pi_{k+1} = B_{k}\pi_{k}$ initialized with $\pi_{1} = \frac{1}{n}\mathbf{1}$. Then, $[\pi_{k}]_{i} \ge b^{n}/n$ for all $k\ge1$ and $i\in[n]$, where $b$ is the lower bound of positive entries of $B_{k}$.
\end{lemma}

\begin{lemma}\label{lemma_sigma_AB}
There exist matrix norms $\|\cdot\|_{\text{A}}$ and $\|\cdot\|_{\text{B}}$ such that $\sigma_{\text{A},k} := \|A_{k} - \mathbf{1}\phi_{k}^T\|_{\text{A}} < 1$ and $\sigma_{\text{B},k} := \|B_{k} - \pi_{k}\mathbf{1}^{T}\|_{\text{B}} < 1$ for all $k\ge1$.
\end{lemma}
\begin{proof}
It follows from Lemma 3 in \cite{pu2020push} that the spectral radii of $\tilde{A}_{k} := A_{k} - \mathbf{1}\phi_{k}^T$ and $\tilde{B}_{k} := B_{k} - \pi_{k}\mathbf{1}^{T}$, denoted as $\rho_{\text{A},k}$ and $\rho_{\text{B},k}$, are less than 1 for all $k\ge1$. Given an arbitrarily small positive number $\epsilon > 0$, by Lemma 5.6.10 in \cite{horn2012matrix}, there exist matrix norms $\|\cdot\|_{\text{A},k}$ such that $\|\tilde{A}_{k}\|_{\text{A},k} \le \rho_{\text{A},k}+\epsilon$. The matrix norm $\|\cdot\|_{\text{A},k}$ for $\tilde{A}_{k}$ can be defined as $\|\tilde{A}_{k}\|_{\text{A},k} := \|(V_{k}U_{k}^{*})\tilde{A}_{k}(V_{k}U_{k}^{*})^{-1}\|_{1}$, where $\|\cdot\|_{1}$ is the maximum column-sum matrix norm, $U_{k}\in\mathbb{R}^{n\times n}$ is the unitary matrix with respect to $\tilde{A}_{k}$, and $V_{k} := \text{diag}\{v_{k},v_{k}^{2},\cdots, v_{k}^{n} \}\in\mathbb{R}^{n\times n}$ with $v_{k}>0$ large enough. If we choose $\tilde{v}:=\max_{k\ge1}\{v_{k}\}$, there exists a matrix norm $\|\cdot\|_{\text{A}}$ defined as $\|\tilde{A}_{k}\|_{\text{A}} := \|(\tilde{V}U_{k}^{*})\tilde{A}_{k}(\tilde{V}U_{k}^{*})^{-1}\|_{1}$ with $\tilde{V}:= \text{diag}\{\tilde{v},\tilde{v}^{2},\cdots, \tilde{v}^{n} \}$ such that $\|\tilde{A}_{k}\|_{\text{A}} \le \rho_{\text{A},k}+\epsilon$ for all $k\ge1$. Since $\rho_{\text{A},k}<1$ for all $k\ge1$ and $\epsilon$ is an arbitrarily small positive number, we have $\|\tilde{A}_{k}\|_{\text{A}} < 1$ for all $k\ge1$, i.e., $\sigma_{\text{A},k} := \|A_{k} - \mathbf{1}\phi_{k}^T\|_{\text{A}} < 1$ for all $k\ge1$. For $\sigma_{\text{B},k} < 1$, the proof is similar to $\sigma_{\text{A},k}$.
\end{proof}

Under Assumptions \ref{assumption_smooth} and \ref{assumption_convex}, the objective function $f_{i}$ is smooth and strongly convex, which aids us in obtaining the following useful lemma.
\begin{lemma}\label{lemma_smooth_convex}(\cite[Lemma 2]{pu2020push} and \cite[Lemma 10]{qu2018harnessing})
Define $g^{k} := \lambda_{k}\mathbf{1}^{T}\nabla F(\mathbf{1}\bar{x}^{k})$ and $\tilde{\alpha}_{k} := \phi_{k}^{T}\bm{\alpha}\pi_{k}$. Then, we have
\begin{align}
&\| \hat{y}^{k}-g^{k} \|_{2} \le \sqrt{n}L\lambda_{k}\|\mathbf{x}^{k}-\mathbf{1}\bar{x}^{k}\|_{2}, \\
&\|g^{k}\|_{2} \le \hat{L}\lambda_{k}\|\bar{x}^{k} - x^{*}\|_{2},
\end{align}
and when $\tilde{\alpha}_{k}\lambda_{k} \le 2/(\hat{\mu}+\hat{L})$ for all $k\ge1$,
\begin{equation}
\begin{aligned}
\| \bar{x}^{k} - \tilde{\alpha}_{k}g^{k} - x^{*} \|_{2} \le (1-\hat{\mu}\tilde{\alpha}_{k}\lambda_{k})\| \bar{x}^{k}-x^{*} \|_{2}.
\end{aligned}
\end{equation}
\end{lemma}

Based on the equivalence of norms and Definition \ref{definition_x}, we have the following lemma:
\begin{lemma}\label{lemma_equivalence_norms}\cite[Lemma 5 and 6]{pu2020push}
For any $\mathbf{x}\in\mathbb{R}^{n\times p}$, we have $\|\mathbf{x}\|_{\text{A}}\le \delta_{\text{A},\text{B}}\|\mathbf{x}\|_{\text{B}}$, $\|\mathbf{x}\|_{\text{B}} \le \delta_{\text{B},2}\|\mathbf{x}\|_{2}$ where $\delta_{\text{A},\text{B}}$ and $\delta_{\text{B},2}$ are positive constants. With a proper rescaling of the norms $\|\cdot\|_{\text{A}}$ and $\|\cdot\|_{\text{B}}$, we have $\|\mathbf{x}\|_{2} \le \|\mathbf{x}\|_{\text{A}}$ and $\|\mathbf{x}\|_{2} \le \|\mathbf{x}\|_{\text{B}}$. For any $M\in\mathbb{R}^{n\times n}$ and $\mathbf{x}\in\mathbb{R}^{n\times p}$, we have $\|M\mathbf{x}\| \le \|M\|\|\mathbf{x}\|$. For any $u\in\mathbb{R}^{n\times 1}$ and $v\in\mathbb{R}^{1\times p}$, we have $\|uv\| \le \|u\|\|v\|_{2}$.
\end{lemma}


In the following lemma, we establish the linear system of inequalities for $\|\bar{x}^{k} - x^{*}\|_{2}$, $\|\mathbf{x}^{k}-\mathbf{1}\bar{x}^{k}\|_{\text{A}}$, and $\|\mathbf{y}^{k} - \pi_{k}\hat{y}^{k}\|_{\text{B}}$.
\begin{lemma}\label{lemma_three_inequalities}
Let Assumptions \ref{assumption_smooth}, \ref{assumption_convex}, \ref{assumption_connected}, \ref{assumption_A}, and \ref{assumption_B} hold. When $\tilde{\alpha}_{k}\lambda_{k} \le 2/(\hat{\mu}+\hat{L})$ holds for all $k\ge1$, we have the following three inequalities: for all $k\ge1$,
\begin{equation}\label{inequality_x_bar}
\begin{aligned}
&\| \bar{x}^{k+1} - x^{*} \|_{2} \\
\le &(1-\hat{\mu}\tilde{\alpha}_{k}\lambda_{k})\| \bar{x}^{k} - x^{*} \|_{2} + \sqrt{n}L\tilde{\alpha}_{k}\lambda_{k}\|\mathbf{x}^{k}-\mathbf{1}\bar{x}^{k}\|_{\text{A}} \\
&+ \check{\alpha}\|\phi_{k}\|_{2}\|\mathbf{y}^{k} - \pi_{k}\hat{y}^{k}\|_{\text{B}},
\end{aligned}
\end{equation}
\begin{equation}\label{inequality_x}
\begin{aligned}
& \|\mathbf{x}^{k+1}-\mathbf{1}\bar{x}^{k+1}\|_{\text{A}} \\
\le & \check{\alpha}\hat{L}\sigma_{\text{A},k}\|\pi_{k}\|_{\text{A}}\lambda_{k}\|\bar{x}^{k} - x^{*}\|_{2} + \check{\alpha}\delta_{\text{A},\text{B}}\sigma_{\text{A},k}\|\mathbf{y}^{k} - \pi_{k}\hat{y}^{k}\|_{\text{B}} \\
&+ \sigma_{\text{A},k}(1 + \sqrt{n}\check{\alpha}L\|\pi_{k}\|_{\text{A}}\lambda_{k})\|\mathbf{x}^{k} - \mathbf{1}\bar{x}^{k}\|_{\text{A}},
\end{aligned}
\end{equation}
\begin{equation}\label{inequality_y}
\begin{aligned}
&\|\mathbf{y}^{k+1} - \pi_{k+1}\hat{y}^{k+1}\|_{\text{B}} \\
\le & \sqrt{n}L\delta_{\text{B},2}\xi_{k+1}(\sqrt{n}\check{\alpha}L\lambda_{k}\lambda_{k+1}\|A_{k}\|_{2}\|\pi_{k}\|_{2} \\
&\quad + \lambda_{k} - \lambda_{k+1})\|\bar{x}^{k} - x^{*}\|_{2} \\
&+ L\delta_{\text{B},2}\xi_{k+1}(\lambda_{k+1}\|A_{k}-I\|_{2} + \lambda_{k} - \lambda_{k+1} \\
&\quad + \sqrt{n}\check{\alpha}L\lambda_{k}\lambda_{k+1}\|A_{k}\|_{2}\|\pi_{k}\|_{2})\|\mathbf{x}^{k}-\mathbf{1}\bar{x}^{k}\|_{\text{A}} \\
&+ (\sigma_{\text{B},k} + \check{\alpha}L\delta_{\text{B},2}\xi_{k+1}\lambda_{k+1}\|A_{k}\|_{2})\|\mathbf{y}^{k} - \pi_{k}\hat{y}^{k}\|_{\text{B}} \\
&+ \delta_{\text{B},2}\xi_{k+1}(\lambda_{k} - \lambda_{k+1})\|\nabla F(\mathbf{1}x^{*})\|_{2},
\end{aligned}
\end{equation}
where $\check{\alpha} = \max_{i\in[n]}\{\alpha_{i}\}$ and $\xi_{k+1} = \|I - \pi_{k+1}\mathbf{1}^{T}\|_{\text{B}}$.
\end{lemma}
\begin{proof}
We first prove the first inequality (\ref{inequality_x_bar}). By the definition of $\bar{x}^{k}$, (\ref{our_compact_algorithm_x}), and Lemma \ref{lemma_phi}, we have
\begin{equation}\label{equality_x_bar}
\begin{aligned}
\bar{x}^{k+1} &= \bar{x}^{k} - \phi_{k}^{T}\bm{\alpha}\mathbf{y}^{k}.
\end{aligned}
\end{equation}
Further, it follows that
\begin{equation*}
\begin{aligned}
\bar{x}^{k+1} &= \bar{x}^{k} - \phi_{k}^{T}\bm{\alpha}(\mathbf{y}^{k} - \pi_{k}\hat{y}^{k} + \pi_{k}\hat{y}^{k}) \\
&= \bar{x}^{k} - \tilde{\alpha}_{k}\hat{y}^{k} - \phi_{k}^{T}\bm{\alpha}(\mathbf{y}^{k} - \pi_{k}\hat{y}^{k}) \\
&= \bar{x}^{k} - \tilde{\alpha}_{k}g^{k} - \tilde{\alpha}_{k}(\hat{y}^{k} - g^{k}) - \phi_{k}^{T}\bm{\alpha}(\mathbf{y}^{k} - \pi_{k}\hat{y}^{k}),
\end{aligned}
\end{equation*}
where $\tilde{\alpha}_{k} = \phi_{k}^{T}\bm{\alpha}\pi_{k}$ and $g^{k} = \lambda_{k}\mathbf{1}^{T}\nabla F(\mathbf{1}\bar{x}^{k})$. Then, in light of Lemmas \ref{lemma_smooth_convex} and \ref{lemma_equivalence_norms}, we can obtain
\begin{equation}
\begin{aligned}
&\| \bar{x}^{k+1} - x^{*} \|_{2} \\
= &\| \bar{x}^{k} - x^{*} - \tilde{\alpha}_{k}g^{k} - \tilde{\alpha}_{k}(\hat{y}^{k} - g^{k}) - \phi_{k}^{T}\bm{\alpha}(\mathbf{y}^{k} - \pi_{k}\hat{y}^{k}) \|_{2} \\
\le & \| \bar{x}^{k} - \tilde{\alpha}_{k}g^{k} - x^{*} \|_{2} + \tilde{\alpha}_{k}\| \hat{y}^{k} - g^{k} \|_{2} \\
&+ \| \phi_{k}^{T}\bm{\alpha}(\mathbf{y}^{k} - \pi_{k}\hat{y}^{k}) \|_{2} \\
\le & (1-\hat{\mu}\tilde{\alpha}_{k}\lambda_{k})\| \bar{x}^{k} - x^{*} \|_{2} + \sqrt{n}L\tilde{\alpha}_{k}\lambda_{k}\|\mathbf{x}^{k}-\mathbf{1}\bar{x}^{k}\|_{\text{A}} \\
&+ \check{\alpha}\|\phi_{k}\|_{2}\|\mathbf{y}^{k} - \pi_{k}\hat{y}^{k}\|_{\text{B}},
\end{aligned}
\end{equation}
where $\check{\alpha} = \max_{i\in[n]}\{\alpha_{i}\}$. 

For the second inequality (\ref{inequality_x}), we first establish the following equality based on (\ref{our_compact_algorithm_x}), (\ref{equality_x_bar}), $A_{k}\mathbf{1} = \mathbf{1}$, and $\phi_{k}^{T}\mathbf{1}=1$:
\begin{equation}
\begin{aligned}
&\mathbf{x}^{k+1}-\mathbf{1}\bar{x}^{k+1} \\
=& A_{k}(\mathbf{x}^{k} - \bm{\alpha}\mathbf{y}^{k}) - \mathbf{1}(\bar{x}^{k} - \phi_{k}^{T}\bm{\alpha}\mathbf{y}^{k}) \\
=& A_{k}(\mathbf{x}^{k} - \mathbf{1}\bar{x}^{k}) - A_{k}\bm{\alpha}\mathbf{y}^{k} + \mathbf{1}\phi_{k}^{T}\bm{\alpha}\mathbf{y}^{k} \\
=& (A_{k} - \mathbf{1}\phi_{k}^{T})(\mathbf{x}^{k} - \mathbf{1}\bar{x}^{k}) - (A_{k} - \mathbf{1}\phi_{k}^{T})\bm{\alpha}\mathbf{y}^{k}.
\end{aligned}
\end{equation}
By Lemmas \ref{lemma_smooth_convex} and \ref{lemma_equivalence_norms}, we can obtain
\begin{equation}
\begin{aligned}
&\|\mathbf{x}^{k+1}-\mathbf{1}\bar{x}^{k+1}\|_{\text{A}} \\
=& \|(A_{k} - \mathbf{1}\phi_{k}^{T})(\mathbf{x}^{k} - \mathbf{1}\bar{x}^{k}) - (A_{k} - \mathbf{1}\phi_{k}^{T})\bm{\alpha}\mathbf{y}^{k}\|_{\text{A}} \\
\le & \sigma_{\text{A},k}\|\mathbf{x}^{k} - \mathbf{1}\bar{x}^{k}\|_{\text{A}} + \sigma_{\text{A},k}\|\bm{\alpha}\mathbf{y}^{k}\|_{\text{A}} \\
\le & \sigma_{\text{A},k}\|\mathbf{x}^{k} - \mathbf{1}\bar{x}^{k}\|_{\text{A}} + \check{\alpha}\sigma_{\text{A},k}\|\mathbf{y}^{k} - \pi_{k}\hat{y}^{k} + \pi_{k}\hat{y}^{k}\|_{\text{A}} \\
\le & \sigma_{\text{A},k}\|\mathbf{x}^{k} - \mathbf{1}\bar{x}^{k}\|_{\text{A}} + \check{\alpha}\sigma_{\text{A},k}\|\mathbf{y}^{k} - \pi_{k}\hat{y}^{k}\|_{\text{A}} \\
&+ \check{\alpha}\sigma_{\text{A},k}\|\pi_{k}\|_{\text{A}}\|\hat{y}^{k} - g^{k} + g^{k}\|_{2} \\
\le & \sigma_{\text{A},k}\|\mathbf{x}^{k} - \mathbf{1}\bar{x}^{k}\|_{\text{A}} + \check{\alpha}\delta_{\text{A},\text{B}}\sigma_{\text{A},k}\|\mathbf{y}^{k} - \pi_{k}\hat{y}^{k}\|_{\text{B}} \\
&+ \check{\alpha}\sigma_{\text{A},k}\|\pi_{k}\|_{\text{A}}(\sqrt{n}L\lambda_{k}\|\mathbf{x}^{k}-\mathbf{1}\bar{x}^{k}\|_{2} + \hat{L}\lambda_{k}\|\bar{x}^{k} - x^{*}\|_{2}) \\
\le& \check{\alpha}\hat{L}\sigma_{\text{A},k}\|\pi_{k}\|_{\text{A}}\lambda_{k}\|\bar{x}^{k} - x^{*}\|_{2} + \check{\alpha}\delta_{\text{A},\text{B}}\sigma_{\text{A},k}\|\mathbf{y}^{k} - \pi_{k}\hat{y}^{k}\|_{\text{B}} \\
&+ \sigma_{\text{A},k}(1 + \sqrt{n}\check{\alpha}L\|\pi_{k}\|_{\text{A}}\lambda_{k})\|\mathbf{x}^{k} - \mathbf{1}\bar{x}^{k}\|_{\text{A}},
\end{aligned}
\end{equation}
where $\sigma_{\text{A},k} = \|A_{k} - \mathbf{1}\phi_{k}^T\|_{\text{A}}$.

For the third inequality (\ref{inequality_y}), based on the definition of $\hat{y}^{k}$, (\ref{our_compact_algorithm_y}), and $\mathbf{1}^{T}B_{k} = \mathbf{1}^{T}$, we have
\begin{equation}
\begin{aligned}
\hat{y}^{k+1} =& \mathbf{1}^{T}\mathbf{y}^{k+1} \\
=& \mathbf{1}^{T}\big(B_{k}\mathbf{y}^{k} + \lambda_{k+1}\nabla F(\mathbf{x}^{k+1}) - \lambda_{k}\nabla F(\mathbf{x}^{k})\big) \\
=& \hat{y}^{k} + \lambda_{k+1}\mathbf{1}^{T}\nabla F(\mathbf{x}^{k+1}) - \lambda_{k}\mathbf{1}^{T}\nabla F(\mathbf{x}^{k}).
\end{aligned}
\end{equation}
By the above equality and Lemma \ref{lemma_pi}, we can establish the relation between $\mathbf{y}^{k+1} - \pi_{k+1}\hat{y}^{k+1}$ and $\mathbf{y}^{k} - \pi_{k}\hat{y}^{k}$ as follows:
\begin{equation}\label{equality_y}
\begin{aligned}
&\mathbf{y}^{k+1} - \pi_{k+1}\hat{y}^{k+1} \\
= & B_{k}\mathbf{y}^{k} + \lambda_{k+1}\nabla F(\mathbf{x}^{k+1}) - \lambda_{k}\nabla F(\mathbf{x}^{k}) \\
&- \pi_{k+1}\big(\hat{y}^{k} + \lambda_{k+1}\mathbf{1}^{T}\nabla F(\mathbf{x}^{k+1}) - \lambda_{k}\mathbf{1}^{T}\nabla F(\mathbf{x}^{k})\big) \\
= & (B_{k} - \pi_{k}\mathbf{1}^{T})(\mathbf{y}^{k} - \pi_{k}\hat{y}^{k}) \\
&+ (I - \pi_{k+1}\mathbf{1}^{T})\big(\lambda_{k+1}\nabla F(\mathbf{x}^{k+1}) - \lambda_{k}\nabla F(\mathbf{x}^{k})\big) \\
= & (B_{k} - \pi_{k}\mathbf{1}^{T})(\mathbf{y}^{k} - \pi_{k}\hat{y}^{k}) \\
&+ (I - \pi_{k+1}\mathbf{1}^{T})\big(\lambda_{k+1}(\nabla F(\mathbf{x}^{k+1}) - \nabla F(\mathbf{x}^{k})) \\
&\quad- (\lambda_{k}-\lambda_{k+1})\nabla F(\mathbf{x}^{k})\big).
\end{aligned}
\end{equation}
Then, it follows from (\ref{equality_y}), (\ref{our_compact_algorithm_x}), and Lemma \ref{lemma_equivalence_norms} that
\begin{equation}\label{inequality_y_k_0}
\begin{aligned}
&\|\mathbf{y}^{k+1} - \pi_{k+1}\hat{y}^{k+1}\|_{\text{B}} \\
\le& \sigma_{\text{B},k} \|\mathbf{y}^{k} - \pi_{k}\hat{y}^{k}\|_{\text{B}} + \delta_{\text{B},2}\xi_{k+1}(\lambda_{k} - \lambda_{k+1})\|\nabla F(\mathbf{x}^{k})\|_{2} \\
&+ \delta_{\text{B},2}\xi_{k+1}\lambda_{k+1}\|\nabla F(\mathbf{x}^{k+1}) - \nabla F(\mathbf{x}^{k})\|_{2} \\
\le & \sigma_{\text{B},k} \|\mathbf{y}^{k} - \pi_{k}\hat{y}^{k}\|_{\text{B}} + \delta_{\text{B},2}\xi_{k+1}(\lambda_{k} - \lambda_{k+1})\|\nabla F(\mathbf{x}^{k})\|_{2} \\
&+ \delta_{\text{B},2}L\xi_{k+1}\lambda_{k+1}\|\mathbf{x}^{k+1} - \mathbf{x}^{k}\|_{2} \\
= & \sigma_{\text{B},k} \|\mathbf{y}^{k} - \pi_{k}\hat{y}^{k}\|_{\text{B}} + \delta_{\text{B},2}\xi_{k+1}(\lambda_{k} - \lambda_{k+1})\|\nabla F(\mathbf{x}^{k})\|_{2} \\
&+ \delta_{\text{B},2}L\xi_{k+1}\lambda_{k+1}\|(A_{k}-I)(\mathbf{x}^{k}-\mathbf{1}\bar{x}^{k}) - A_{k}\bm{\alpha}\mathbf{y}^{k}\|_{2} \\
\le & \sigma_{\text{B},k} \|\mathbf{y}^{k} - \pi_{k}\hat{y}^{k}\|_{\text{B}} + \delta_{\text{B},2}\xi_{k+1}(\lambda_{k} - \lambda_{k+1})\|\nabla F(\mathbf{x}^{k})\|_{2} \\
&+ \delta_{\text{B},2}L\xi_{k+1}\lambda_{k+1}\|A_{k}-I\|_{2}\|\mathbf{x}^{k}-\mathbf{1}\bar{x}^{k}\|_{2} \\
&+ \delta_{\text{B},2}L\xi_{k+1}\lambda_{k+1}\|A_{k}\bm{\alpha}(\mathbf{y}^{k} - \pi_{k}\hat{y}^{k}) + A_{k}\bm{\alpha}\pi_{k}\hat{y}^{k}\|_{2}\\
\le & \sigma_{\text{B},k} \|\mathbf{y}^{k} - \pi_{k}\hat{y}^{k}\|_{\text{B}} + \delta_{\text{B},2}\xi_{k+1}(\lambda_{k} - \lambda_{k+1})\|\nabla F(\mathbf{x}^{k})\|_{2} \\
&+ \delta_{\text{B},2}L\xi_{k+1}\lambda_{k+1}\|A_{k}-I\|_{2}\|\mathbf{x}^{k}-\mathbf{1}\bar{x}^{k}\|_{2} \\
&+ \delta_{\text{B},2}\check{\alpha}L\xi_{k+1}\lambda_{k+1}\|A_{k}\|_{2}\|\mathbf{y}^{k} - \pi_{k}\hat{y}^{k}\|_{2}\\
&+ \delta_{\text{B},2}\check{\alpha}L\xi_{k+1}\lambda_{k+1}\|A_{k}\|_{2}\|\pi_{k}\|_{2}\|\hat{y}^{k}\|_{2},
\end{aligned}
\end{equation}
where $\sigma_{\text{B},k} = \|B_{k} - \pi_{k}\mathbf{1}^{T}\|_{\text{B}}$ and $\xi_{k+1} = \|I - \pi_{k+1}\mathbf{1}^{T}\|_{\text{B}}$. In addition, it follows from Lemma \ref{lemma_smooth_convex} that
\begin{equation}\label{inequality_y_hat}
\begin{aligned}
\|\hat{y}^{k}\|_{2} =& \|\hat{y}^{k} - g^{k} + g^{k}\|_{2} \\
\le & \sqrt{n}L\lambda_{k}\|\mathbf{x}^{k}-\mathbf{1}\bar{x}^{k}\|_{2} + \hat{L}\lambda_{k}\|\bar{x}^{k} - x^{*}\|_{2}.
\end{aligned}
\end{equation}
\begin{small}
\begin{equation}\label{inequality_nabla_F}
\begin{aligned}
&\|\nabla F(\mathbf{x}^{k})\|_{2} \\
=& \|\nabla F(\mathbf{x}^{k}) - \nabla F(\mathbf{1}\bar{x}^{k}) + \nabla F(\mathbf{1}\bar{x}^{k}) - \nabla F(\mathbf{1}x^{*}) + \nabla F(\mathbf{1}x^{*})\|_{2} \\
\le & \|\nabla F(\mathbf{x}^{k}) - \nabla F(\mathbf{1}\bar{x}^{k})\|_{2} + \|\nabla F(\mathbf{1}\bar{x}^{k}) - \nabla F(\mathbf{1}x^{*})\|_{2} \\
&+ \|\nabla F(\mathbf{1}x^{*})\|_{2} \\
\le & L\|\mathbf{x}^{k} - \mathbf{1}\bar{x}^{k}\|_{2} + \sqrt{n}L\|\bar{x}^{k} - x^{*}\|_{2} + \|\nabla F(\mathbf{1}x^{*})\|_{2}.
\end{aligned}
\end{equation}
\end{small}
Finally, combining inequalities (\ref{inequality_y_k_0}), (\ref{inequality_y_hat}), and (\ref{inequality_nabla_F}), we obtain the third inequality.
\end{proof}

\subsection{Convergence Results}
In this subsection, we first establish a linear system of inequalities with respect to $\|\bar{x}^{k+1} - x^{*}\|_{2}$, $\|\mathbf{x}^{k+1}-\mathbf{1}\bar{x}^{k+1}\|_{\text{A}}$, $\|\mathbf{y}^{k+1} - \pi_{k+1}\hat{y}^{k+1}\|_{\text{B}}$ and their previous values, and then give the convergence result in Theorem \ref{theorem_convergence_rate}.

Define $s_{1}^{k} := \|\bar{x}^{k} - x^{*}\|_{2}$, $s_{2}^{k} := \|\mathbf{x}^{k}-\mathbf{1}\bar{x}^{k}\|_{\text{A}}$, and $s_{3}^{k} := \|\mathbf{y}^{k} - \pi_{k}\hat{y}^{k}\|_{\text{B}}$. Let $\bm{s}^{k} := \left(s_{1}^{k}, s_{2}^{k}, s_{3}^{k}\right)^{T}$. In light of Lemma \ref{lemma_three_inequalities}, we have
\begin{equation}\label{inequality_linear_system}
\bm{s}^{k+1} \le C_{k}\bm{s}^{k} + \mathbf{d}_{k},
\end{equation}
where
\begin{equation*}
\begin{aligned}
&[C_{k}]_{11} = 1-\hat{\mu}\tilde{\alpha}_{k}\lambda_{k},\  [C_{k}]_{12} = \sqrt{n}L\tilde{\alpha}_{k}\lambda_{k}, \\
&[C_{k}]_{13} = \check{\alpha}\|\phi_{k}\|_{2},\ [C_{k}]_{21} = \check{\alpha}\hat{L}\sigma_{\text{A},k}\|\pi_{k}\|_{\text{A}}\lambda_{k}, \\
&[C_{k}]_{22} = \sigma_{\text{A},k}(1 + \sqrt{n}\check{\alpha}L\|\pi_{k}\|_{\text{A}}\lambda_{k}),\ [C_{k}]_{23} = \check{\alpha}\delta_{\text{A},\text{B}}\sigma_{\text{A},k}, \\
&[C_{k}]_{31} = \sqrt{n}L\delta_{\text{B},2}\xi_{k+1}(\sqrt{n}\check{\alpha}L\lambda_{k}\lambda_{k+1}\|A_{k}\|_{2}\|\pi_{k}\|_{2} \\
&\qquad\qquad+ \lambda_{k} - \lambda_{k+1}), \\
&[C_{k}]_{32} = L\delta_{\text{B},2}\xi_{k+1}(\sqrt{n}\check{\alpha}L\lambda_{k}\lambda_{k+1}\|A_{k}\|_{2}\|\pi_{k}\|_{2} \\
&\qquad\qquad+ \lambda_{k+1}\|A_{k}-I\|_{2} + \lambda_{k} - \lambda_{k+1}), \\
&[C_{k}]_{33} = \sigma_{\text{B},k} + \check{\alpha}L\delta_{\text{B},2}\xi_{k+1}\lambda_{k+1}\|A_{k}\|_{2}, \\
&[\mathbf{d}_{k}]_{1} = [\mathbf{d}_{k}]_{2} = 0, \\
&[\mathbf{d}_{k}]_{3} = \delta_{\text{B},2}\xi_{k+1}(\lambda_{k} - \lambda_{k+1})\|\nabla F(\mathbf{1}x^{*})\|_{2}.
\end{aligned}
\end{equation*}

We also need the following two lemmas that will help us derive our main result.

\begin{lemma}\label{lemma_determine_radius}
\cite[Lemma 5]{pu2021distributed} Given a non-negative and irreducible matrix $M\in\mathbb{R}^{3\times 3}$ with $[M]_{ii} < c^{*}$ for some $c^{*}>0$ and all $i \in [3]$. Then, $\rho(M) < c^{*}$ if and only if $\det(c^{*}I-M)>0$.
\end{lemma}

\begin{lemma}\label{lemma:s_1}
If the convergent non-negative sequences $\lambda_{k}$ and $r_{k}$ satisfies $\sum_{k=1}^{\infty}\lambda_{k} = \infty$, $1 - c\lambda_{k} \ge 0$ with $c>0$ for all $k\ge1$, and $\lim_{k\to\infty}r_{k} = 0$, then the followings hold:
	\begin{enumerate}
	\item[\textit{(a)}] $\lim_{k\to\infty}\prod_{i=1}^{k}(1-c\lambda_{i}) = 0$.
	\item[\textit{(b)}] $\lim_{k\to\infty}\sum_{i=1}^{k}\prod_{j=i+1}^{k}(1-c\lambda_{j})$ is bounded.
	\item[\textit{(c)}] $\lim_{k\to\infty}\sum_{i=1}^{k}\big(\prod_{j=i+1}^{k}(1-c\lambda_{j})\big)r_{i} = 0$.
	\end{enumerate}
\end{lemma}
\begin{proof}
Due to $1-x\le e^{-x}$ for $x\in\mathbb{R}$, it follows that
	\begin{equation}
	\begin{aligned}
	\lim_{k\to\infty}\prod_{i=1}^{k}(1-c\lambda_{i}) &\le \prod_{i=1}^{\infty}e^{-c\lambda_{i}} = e^{-c\sum_{i=1}^{\infty}\lambda_{i}} = 0.
	\end{aligned}
	\end{equation}
Since $1-c\lambda_{k}$ is non-negative, we have the part \textit{(a)}.

Let $a_{i}=\lim_{k\to\infty}\prod_{j=k-i+2}^{k}(1-c\lambda_{j})$ for $i\le k$. Since $\lim_{i\to\infty} a_{i} = 0$ and $\lim_{i\to\infty}\frac{a_{i+1}}{a_{i}} = 1- c\lambda_{1} < 1$,
we have $\sum_{i=1}^{\infty}a_{i} = \lim_{k\to\infty}\sum_{i=1}^{k}\prod_{j=i+1}^{k}(1-c\lambda_{j}) < \infty$. Hence, the part \textit{(b)} holds.

Let $\epsilon>0$ be arbitrary. Since the convergent sequence $r_{k}$ satisfies $\lim_{k\to\infty}r_{k} = 0$, there exists $K$ such that $r_{k}\le\epsilon$ for all $k\ge K$. For $k\ge K+1$, we have
	\begin{equation}
	\begin{aligned}
	&\sum_{i=1}^{k}\Big(\prod_{j=i+1}^{k}(1-c\lambda_{j})\Big) r_{i} \\
	= &\sum_{i=1}^{K}\Big(\prod_{j=i+1}^{k}(1-c\lambda_{j})\Big) r_{i} + \sum_{i=K+1}^{k}\Big(\prod_{j=i+1}^{k}(1-c\lambda_{j})\Big) r_{i} \\
	\le &\max_{1\le l\le K}r_{l}\sum_{i=1}^{K}\prod_{j=i+1}^{k}(1-c\lambda_{j}) + \epsilon\sum_{i=K+1}^{k}\prod_{j=i+1}^{k}(1-c\lambda_{j}).
	\end{aligned}
	\end{equation}
By part \textit{(a)}, we have $\lim_{k\to\infty}\sum_{i=1}^{K}\prod_{j=i+1}^{k}(1-c\lambda_{j}) = 0$. By \textit{(b)}, we have $\lim_{k\to\infty}\sum_{i=K+1}^{k}\prod_{j=i+1}^{k}(1-c\lambda_{j})$ is bounded. Since $\epsilon>0$ is arbitrary, we have $\lim_{k\to\infty}\sum_{i=1}^{k}\big(\prod_{j=i+1}^{k}(1-c\lambda_{j})\big) r_{i}\le0$. Given that $1 - c\lambda_{k}$ and $r_{k}$ are non-negative, we obtain the part \textit{(c)}.
\end{proof}



\begin{theorem}\label{theorem_convergence_rate}
Suppose Assumptions \ref{assumption_smooth}-\ref{assumption_B} hold. When the step size $\alpha_{i}$ satisfies condition (\ref{alpha_check_sum}) and the decaying weight sequence $\lambda_{k}$ satisfies $\sum_{k=1}^{\infty}\lambda_{k}=\infty$, the sequence $x_{i}^{k}$ generated by the update equation (\ref{our_algorithm}) will converge to the optimal solution $x^{*}$.
\end{theorem}
\begin{proof}
Our proof idea is based on the spectral radius of $C_{k}$. Considering $\lambda_{k}\to0$ as $k\to\infty$, we analyze $\rho(C_{k})$ under two cases, $1\le k<\infty$ and $k\to\infty$, respectively.

When $1\le k<\infty$, we will show that $\rho(C_{k})$ is strictly less than 1.
By Lemma \ref{lemma_determine_radius}, it suffices to prove that $[C_{k}]_{ii}<1$ for all $i \in [3]$ and $\det(I-C_{k}) > 0$. If $\tilde{\alpha}_{k}\lambda_{k} \le 2/(\hat{\mu}+\hat{L})$, we have $0 \le [C_{k}]_{11} < 1$. To guarantee $[C_{k}]_{22}<1$ and $[C_{k}]_{33}<1$, we only need $\check{\alpha} < \min\left\{ \frac{1-\sigma_{\text{A},k}}{\sqrt{n}L\sigma_{\text{A},k}\lambda_{k}\|\pi_{k}\|_{\text{A}}}, \frac{1-\sigma_{\text{B},k}}{L\delta_{\text{B},2}\xi_{k+1}\lambda_{k+1}\|A_{k}\|_{2}} \right\}$.

If $\check{\alpha}$ further satisfies
\begin{equation}\label{check_alpha_11}
\begin{aligned}
\check{\alpha} \le \min\Big\{ \frac{1-\sigma_{\text{A},k}}{2\sqrt{n}L\sigma_{\text{A},k}\lambda_{k}\|\pi_{k}\|_{\text{A}}}, \frac{1-\sigma_{\text{B},k}}{2L\delta_{\text{B},2}\xi_{k+1}\lambda_{k+1}\|A_{k}\|_{2}} \Big\},
\end{aligned}
\end{equation}
then $1-[C_{k}]_{22} \ge (1-\sigma_{\text{A},k})/2$ and $1-[C_{k}]_{33} \ge (1-\sigma_{\text{B},k})/2$.

We now need to make sure that $\det(I-C_{k}) > 0$ holds. Let $\tilde{\alpha}_{k} = \theta_{k}\check{\alpha}$ with $\theta_{k}>0$. Then, we have
\begin{equation}\label{equation_det}
\begin{aligned}
&\det(I-C_{k}) \\
= &(1-[C_{k}]_{11})(1-[C_{k}]_{22})(1-[C_{k}]_{33}) \\
&- [C_{k}]_{12}[C_{k}]_{23}[C_{k}]_{31} - [C_{k}]_{13}[C_{k}]_{21}[C_{k}]_{32}\\
&- [C_{k}]_{13}[C_{k}]_{31}(1-[C_{k}]_{22}) - [C_{k}]_{23}[C_{k}]_{32}(1-[C_{k}]_{11}) \\
&- [C_{k}]_{12}[C_{k}]_{21}(1-[C_{k}]_{33}) \\
\ge & \frac{\hat{\mu}(1-\sigma_{\text{A},k})(1-\sigma_{\text{B},k})\theta_{k}\lambda_{k}\check{\alpha}}{4} \\
&- nL^{2}\delta_{\text{A},\text{B}}\sigma_{\text{A},k}\delta_{\text{B},2}\xi_{k+1}\theta_{k}\lambda_{k}\check{\alpha}^{2}(\lambda_{k} - \lambda_{k+1} \\
&\quad + \sqrt{n}L\|A_{k}\|_{2}\|\pi_{k}\|_{2}\lambda_{k}\lambda_{k+1}\check{\alpha}) \\
&- nL^{2}\sigma_{\text{A},k}\delta_{\text{B},2}\xi_{k+1}\|\phi_{k}\|_{2}\|\pi_{k}\|_{\text{A}}\lambda_{k}\check{\alpha}^{2}\big((\|A_{k}\|_{2}+1)\lambda_{k+1} \\
&\quad + \lambda_{k} - \lambda_{k+1} + \sqrt{n}L\|A_{k}\|_{2}\|\pi_{k}\|_{2}\lambda_{k}\lambda_{k+1}\check{\alpha}\big) \\
&- \frac{\sqrt{n}L\delta_{\text{B},2}\xi_{k+1}\|\phi_{k}\|_{2}(1-\sigma_{\text{A},k})\check{\alpha}}{2}(\lambda_{k} - \lambda_{k+1} \\
&\quad +\sqrt{n}L\|A_{k}\|_{2}\|\pi_{k}\|_{2}\lambda_{k}\lambda_{k+1}\check{\alpha}) \\
&- \hat{\mu}L\delta_{\text{A},\text{B}}\sigma_{\text{A},k}\delta_{\text{B},2}\xi_{k+1}\theta_{k}\lambda_{k}\check{\alpha}^{2} \big((\|A_{k}\|_{2}+1)\lambda_{k+1} \\
&\quad+ \lambda_{k} - \lambda_{k+1} + \sqrt{n}L\|A_{k}\|_{2}\|\pi_{k}\|_{2}\lambda_{k}\lambda_{k+1}\check{\alpha}\big) \\
&- \frac{n\sqrt{n}L^{2}\sigma_{\text{A},k}\|\pi_{k}\|_{\text{A}}(1-\sigma_{\text{B},k})\theta_{k}\lambda_{k}^{2}\check{\alpha}^{2}}{2}.
\end{aligned}
\end{equation}

By (\ref{equation_det}), it is evident that ensuring $\det(I-C_{k}) > 0$ is equivalent to ensuring $e_{1,k}\check{\alpha}^{2} + e_{2,k}\check{\alpha} - e_{3,k} < 0$, where
\begin{align*}
e_{1,k} =& n\sqrt{n}L^{2}\sigma_{\text{A},k}\delta_{\text{B},2}\xi_{k+1}\|A_{k}\|_{2}\|\pi_{k}\|_{2}\lambda_{k}^{2}\lambda_{k+1} (L\delta_{\text{A},\text{B}}\theta_{k} \\
&+ L\|\phi_{k}\|_{2}\|\pi_{k}\|_{\text{A}} + \mu \delta_{\text{A},\text{B}}\theta_{k}), \\
e_{2,k} =& nL\sigma_{\text{A},k}\delta_{\text{B},2}\xi_{k+1}\lambda_{k}(\lambda_{k} - \lambda_{k+1})\big( (L+\mu)\delta_{\text{A},\text{B}}\theta_{k} \\
&\quad+ L\|\phi_{k}\|_{2}\|\pi_{k}\|_{\text{A}} \big) \\
&+ nL\delta_{\text{B},2}\xi_{k+1}\lambda_{k}\lambda_{k+1}\big( 0.5L\|\phi_{k}\|_{2}(1-\sigma_{\text{A},k})\|A_{k}\|_{2}\|\pi_{k}\|_{2} \\
&\quad+ (L\sigma_{\text{A},k}\|\phi_{k}\|_{2}\|\pi_{k}\|_{\text{A}} + \mu\delta_{\text{A},\text{B}}\sigma_{\text{A},k}\theta_{k})(\|A_{k}\|_{2}+1) \big) \\
&+ 0.5n\sqrt{n}L^{2}\sigma_{\text{A},k}\|\pi_{k}\|_{\text{A}}(1-\sigma_{\text{B},k})\theta_{k}\lambda_{k}^{2}, \\
e_{3,k} =& 0.25\hat{\mu}(1-\sigma_{\text{A},k})(1-\sigma_{\text{B},k})\theta_{k}\lambda_{k} \\
&- 0.5\sqrt{n}L\delta_{\text{B},2}\xi_{k+1}\|\phi_{k}\|_{2}(1-\sigma_{\text{A},k})(\lambda_{k} - \lambda_{k+1}).
\end{align*}


If $\lambda_{k+1}/\lambda_{k}$ and $\check{\alpha}$ satisfy
\begin{align}
&1 - \frac{\sqrt{n}\mu(1-\sigma_{\text{B},k})\theta_{k}}{2L\delta_{\text{B},2}\xi_{k+1}\|\phi_{k}\|_{2}} < \frac{\lambda_{k+1}}{\lambda_{k}} \le 1, \label{inequality_lambdas} \\
&0 < \check{\alpha} < \frac{2e_{3,k}}{e_{2,k} + \sqrt{e_{2,k}^{2}+4e_{1,k}e_{3,k}}}, \label{check_alpha_2}
\end{align}
then $e_{1,k}\check{\alpha}^{2} + e_{2,k}\check{\alpha} - e_{3,k} < 0$. Notably, in (\ref{inequality_lambdas}), since $\frac{\sqrt{n}\mu(1-\sigma_{\text{B},k})\theta_{k}}{2L\delta_{\text{B},2}\xi_{k+1}\|\phi_{k}\|_{2}}>0$, we have $1 - \frac{\sqrt{n}\mu(1-\sigma_{\text{B},k})\theta_{k}}{2L\delta_{\text{B},2}\xi_{k+1}\|\phi_{k}\|_{2}} < 1$. Given that $\sum_{k=1}^{\infty}\lambda_{k}=\infty$ and $\lambda_{k+1}\le\lambda_{k}$, $\lim_{k\to\infty}\frac{\lambda_{k+1}}{\lambda_{k}}=1$. There must exist a $k'$ such that (\ref{inequality_lambdas}) holds for all $k\ge k'$. Hence, we only need $\check{\alpha}$ to satisfy $\check{\alpha} \le \frac{2}{\theta_{k}\lambda_{k}(\hat{\mu}+\hat{L})}$, (\ref{check_alpha_11}), and (\ref{check_alpha_2}) for all $k\ge k'$. Let $k = k - k'$. When $\check{\alpha}$ satisfies
\begin{equation}\label{alpha_check_sum}
\begin{aligned}
0 <& \check{\alpha} < \min_{k\ge1}\Bigg\{ \frac{2}{\theta_{k}\lambda_{k}(\hat{\mu}+\hat{L})}, \frac{1-\sigma_{\text{A},k}}{2\sqrt{n}L\sigma_{\text{A},k}\lambda_{k}\|\pi_{k}\|_{\text{A}}}, \\
&\frac{1-\sigma_{\text{B},k}}{2L\delta_{\text{B},2}\xi_{k+1}\lambda_{k+1}\|A_{k}\|_{2}}, \frac{2e_{3,k}}{e_{2,k} + \sqrt{e_{2,k}^{2}+4e_{1,k}e_{3,k}}} \Bigg\},
\end{aligned}
\end{equation}
then $\rho(C_{k}) < 1$ holds for $1\le k <\infty$.

When $k\to\infty$, $\lambda_{k}\to0$ and
	\begin{equation}\label{eq:C_infty}
C_{\infty} = 
\left[ \begin{matrix}
1 & 0 & \check{\alpha}\|\phi_{\infty}\|_{2} \\
0 & \sigma_{\text{A},\infty} & \check{\alpha}\delta_{\text{A},\text{B}}\sigma_{\text{A},\infty} \\
0 & 0 & \sigma_{\text{B},\infty}
	\end{matrix} \right].
	\end{equation}
Since $C_{\infty}$ is an upper triangular matrix, its eigenvalues are equal to its diagonal elements. Given that $\sigma_{\text{A},\infty}<1$ and $\sigma_{\text{B},\infty}<1$, we have $\rho(C_{\infty}) = 1$. 

By Lemma 5.6.10 in \cite{horn2012matrix}, given an arbitrary $\varepsilon>0$, there exists an induced matrix norm $\|\cdot\|_{\text{C}}$ such that $\|C_{k}\|_{\text{C}} \le \rho(C_{k}) + \varepsilon$ when $1\le k<\infty$ and $\|C_{k}\|_{\text{C}} = 1$ when $k\to\infty$. Since $\rho(C_{k})<1$ when $1\le k<\infty$, we have $\|C_{k}\|_{\text{C}}\le 1$ for all $k\ge1$. Then, recursively applying (\ref{inequality_linear_system}), we have
\begin{equation}\label{ieq:recursive_linear_sys}
\begin{aligned}
\| \bm{s}^{k+1} \|_{\text{C}} &\le \|C_{k}C_{k-1}\cdots C_{1}\bm{s}^{1}\|_{\text{C}} + \sum_{m=1}^{k}\Big\|\Big(\prod_{l=m+1}^{k}C_{l}\Big) \mathbf{d}_{m}\Big\|_{\text{C}} \\
&\le \|\bm{s}^{1}\|_{\text{C}} + \sum_{m=1}^{\infty}\|\mathbf{d}_{m}\|_{\text{C}}.
\end{aligned}
\end{equation}
Based on $\sum_{m=1}^{\infty}[\mathbf{d}_{m}]_{i}<\infty$ for all $i\in[3]$ and the equivalence of norms, we have $\bm{s}^{k}$ is bounded.

Let $s_{1}^{k} \le c_{1}$, $s_{2}^{k} \le c_{2}$, and $s_{3}^{k} \le c_{3}$. Using (\ref{inequality_linear_system}), we have
	\begin{subequations}
	\begin{align}
	s_{1}^{k+1} &\le (1-\hat{\mu}\tilde{\alpha}_{k}\lambda_{k})s_{1}^{k} + c_{2}[C_{k}]_{12} + [C_{k}]_{13}s_{3}^{k}, \label{ieq:s_1}\\
	\tilde{\bm{s}}^{k+1} &\le \tilde{C}_{k}\tilde{\bm{s}}^{k} + \tilde{\mathbf{d}}_{k},
	\end{align}
	\end{subequations}
where $\tilde{\bm{s}}^{k} := [s_{2}^{k}, s_{3}^{k}]^{T}$, $\tilde{C}_{k} = \left[\begin{matrix} \sigma_{\text{A},k} & [C_{k}]_{23} \\ 0 & \sigma_{\text{B},k}\end{matrix} \right]$, and $\tilde{\mathbf{d}}_{k} = \big[c_{1}[C_{k}]_{21} + c_{2}\sqrt{n}\check{\alpha}L\sigma_{\text{A},k}\|\pi_{k}\|_{\text{A}}\lambda_{k}, c_{1}[C_{k}]_{31} + c_{2}[C_{k}]_{32} + c_{3}\check{\alpha}L\delta_{\text{B},2}\xi_{k+1}\lambda_{k+1}\|A_{k}\|_{2} + [\mathbf{d}_{k}]_{3}\big]^{T}$. Clearly, $\rho({\tilde{C}_{k}})<1$. Thus, $\tilde{\bm{s}}^{k}$ is convergent and satisfies $\tilde{\bm{s}}^{\infty} \le (I-\tilde{C}_{\infty})^{-1}\tilde{\mathbf{d}}_{\infty}$. Since $\tilde{\mathbf{d}}_{\infty} = 0$ and $\tilde{\bm{s}}^{k}\ge0$, we have $s_{2}^{\infty} = \lim_{k\to\infty}\|\mathbf{x}^{k}-\mathbf{1}\bar{x}^{k}\|_{\text{A}} = 0$ and $s_{3}^{\infty} = \lim_{k\to\infty}\|\mathbf{y}^{k} - \pi_{k}\hat{y}^{k}\|_{\text{B}}=0$.

By (\ref{ieq:s_1}), we have
	\begin{equation}
	s_{1}^{k+1} \le \Big(\prod_{i=1}^{k}(1-\hat{\mu}\tilde{\alpha}_{i}\lambda_{i})\Big) s_{1}^{1} + \sum_{i=1}^{k} \Big(\prod_{j=i+1}^{k}(1-\hat{\mu}\tilde{\alpha}_{j}\lambda_{j})\Big) d_{i},
	\end{equation}
where $d_{k} = c_{2}[C_{k}]_{12} + [C_{k}]_{13}s_{3}^{k}$. Using Lemma \ref{lemma:s_1}, we have $\lim_{k\to\infty}s_{1}^{k}\le0$. Since $s_{1}^{k}\ge0$, $\lim_{k\to\infty}s_{1}^{k} = \lim_{k\to\infty}\|\bar{x}^{k} - x^{*}\|_{2} = 0$. The proof is complete.
\end{proof}

\begin{figure}[!t]
\begin{center}
  \includegraphics[width = \linewidth]{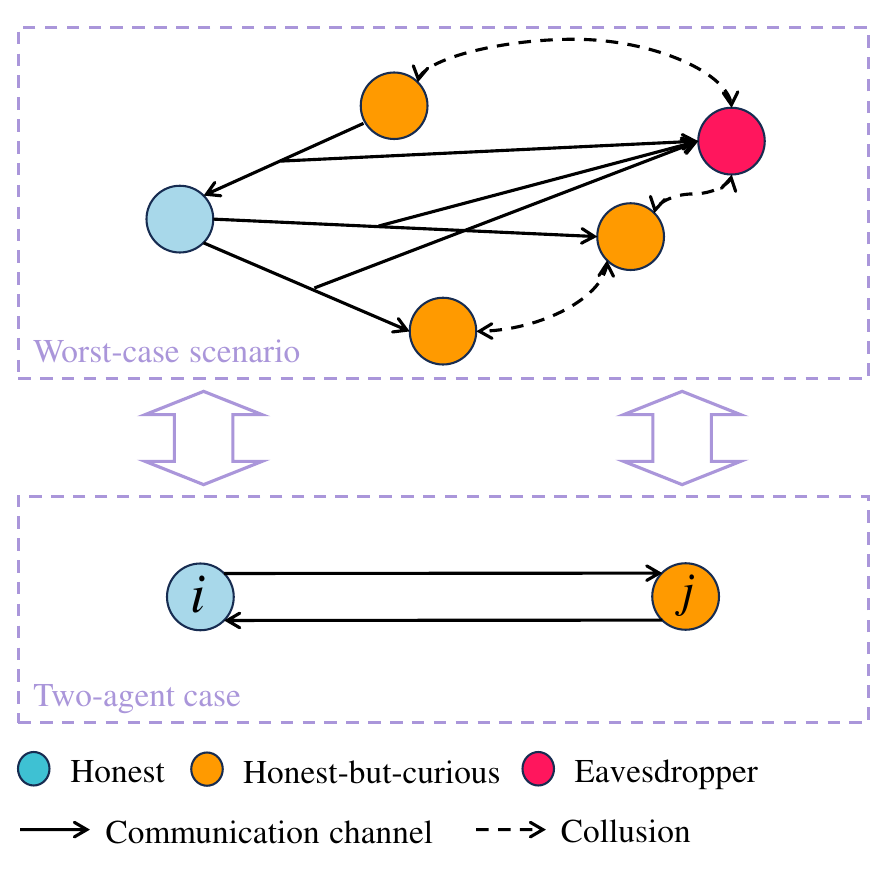}
  \caption{The worst-case scenario of privacy leakage risk for an honest agent and the equivalent two-agent system.}
  \label{fig_two_agents}
\end{center}
\end{figure}
\section{Privacy Analysis}\label{section_privacy}
In this section, we will analyze the privacy preservation performance of the proposed algorithm in detail. In the problem set, it is possible that all of an honest agent's neighbors are honest-but-curious adversaries, and these adversaries can collude with external eavesdroppers to infer the honest agent’s private information. We consider an equivalent two-agent system to simplify the analysis, as depicted in Fig. \ref{fig_two_agents}. Within this system, agents $i$ and $j$ are designated as the honest agent and the attacker, respectively, engaging in bi-directional communication. Attacker $j$ has the abilities of both an external eavesdropper and an honest-but-curious adversary. Specifically, attacker $j$ has access to all exchanged messages on communication channels and also knows agent $i$'s update and communication protocols. Without loss of generality, we assert that an algorithm effectively defends against external eavesdroppers and honest-but-curious adversaries if it successfully protects agent $i$'s private information from being stolen by agent $j$ in this scenario. We have the following results.

\begin{theorem}\label{theorem_privacy}
For the distributed optimization problem (\ref{problem_optimization}), the proposed algorithm can protect the private information of agents, including their intermediate states and all gradients, from being stolen by both external eavesdroppers and honest-but-curious adversaries.
\end{theorem}
\begin{proof}
In the two-agent system shown in Fig. \ref{fig_two_agents}, attacker $j$ has knowledge of the following information:
\begin{equation*}
\begin{aligned}
&I_{x,i}^{k} := x_{i}^{k} - \alpha_{i}y_{i}^{k},\quad I_{x,j}^{k} := x_{j}^{k} - \alpha_{j}y_{j}^{k}, \\
&I_{y,i}^{k} := [B_{k}]_{ji}y_{i}^{k},\quad I_{y,j}^{k} := [B_{k}]_{ij}y_{j}^{k}.
\end{aligned}
\end{equation*}
In addition, attacker $j$ is aware of the update protocol of agent $i$ as detailed below.
\begin{subequations}\label{equations_update_agent_i}
\begin{align}
&x_{i}^{k+1} = [A_{k}]_{ii}(x_{i}^{k} - \alpha_{i}y_{i}^{k}) + [A_{k}]_{ij}(x_{j}^{k} - \alpha_{j}y_{j}^{k}), \\
&y_{i}^{k+1} = [B_{k}]_{ii}y_{i}^{k} + [B_{k}]_{ij}y_{j}^{k} + \lambda_{k+1}\nabla f_{i}(x_{i}^{k+1}) \nonumber\\
&\qquad\quad- \lambda_{k}\nabla f_{i}(x_{i}^{k}).
\end{align}
\end{subequations}

Let $K$ denote the total number of iterations required for the algorithm to converge. To steal the intermediate states of agent $i$ (i.e., $x_{i}^{1}, x_{i}^{2}, \cdots, x_{i}^{K}$), attacker $j$ can construct the following system of equations based on the obtained information and the fact $[A_{k}]_{ii} = 1 - [A_{k}]_{ij}$.
\begin{equation}\label{equation_system_equations_x}
\begin{aligned}
&x_{i}^{2} = I_{x,i}^{1} + [A_{1}]_{ij}(I_{x,j}^{1} - I_{x,i}^{1}), \\
&x_{i}^{3} = I_{x,i}^{2} + [A_{2}]_{ij}(I_{x,j}^{2} - I_{x,i}^{2}), \\
&\ \vdots \\
&x_{i}^{K} = I_{x,i}^{K-1} + [A_{K-1}]_{ij}(I_{x,j}^{K-1} - I_{x,i}^{K-1}).
\end{aligned}
\end{equation}
In the above system of equations, there are $(K-1)p$ equations and $(K-1)(p+1)$ unknown variables (i.e., $x_{i}^{2}, x_{i}^{3}, \cdots, x_{i}^{K}$ and $[A_{1}]_{ij}, [A_{2}]_{ij}, \cdots, [A_{K-1}]_{ij}$). Since the number of unknowns exceeds the number of equations, there exist infinitely many solutions satisfying the above system of equations. Thus, attacker $j$ cannot infer $x_{i}^{2}, x_{i}^{3}, \cdots, x_{i}^{K}$ \citep{zhang2019admm}. In addition, attacker $j$ cannot infer $x_{i}^{1}$ based on $I_{x,i}^{1} := x_{i}^{1} - \alpha_{i}y_{i}^{1}$ due to the unknown variables of $y_{i}^{1}$ and $\alpha_{i}$. Therefore, attacker $j$ cannot steal the intermediate states of agent $i$.

In order to steal agent $i$'s gradient evaluated at any point (i.e., $\nabla f_{i}(x_{i}^{1}), \nabla f_{i}(x_{i}^{2}), \cdots, \nabla f_{i}(x_{i}^{K+1})$), attacker $j$ can build the following system of equations based on the facts $[B_{k}]_{ii} = 1 - [B_{k}]_{ji}$ and $y_{i}^{1} = \lambda_{1}\nabla f_{i}(x_{i}^{1})$.
\begin{equation}\label{equation_system_equations_y}
\begin{aligned}
&y_{i}^{2} = - I_{y,i}^{1} + I_{y,j}^{1} + \lambda_{2}\nabla f_{i}(x_{i}^{2}), \\
&y_{i}^{3} = y_{i}^{2} - I_{y,i}^{2} + I_{y,j}^{2} + \lambda_{3}\nabla f_{i}(x_{i}^{3}) - \lambda_{2}\nabla f_{i}(x_{i}^{2}), \\
&\ \vdots \\
&y_{i}^{K+1} = y_{i}^{K} - I_{y,i}^{K} + I_{y,j}^{K} + \lambda_{K+1}\nabla f_{i}(x_{i}^{K+1}) \\
&\qquad\quad\  - \lambda_{K}\nabla f_{i}(x_{i}^{K}).
\end{aligned}
\end{equation}
There are $Kp$ equations and $(2K-1)p$ unknown variables (i.e., $y_{i}^{2}, \cdots, y_{i}^{K}$ and $\nabla f_{i}(x_{i}^{2}), \cdots, \nabla f_{i}(x_{i}^{K+1})$). Following the similar reasoning as in (\ref{equation_system_equations_x}), attacker $j$ cannot infer $\nabla f_{i}(x_{i}^{2}), \nabla f_{i}(x_{i}^{3}), \cdots, \nabla f_{i}(x_{i}^{K+1})$. In addition, attacker $j$ cannot utilize $I_{x,i}^{1} := x_{i}^{1} - \alpha_{i}\lambda_{1}\nabla f_{i}(x_{i}^{1})$ and $I_{y,i}^{1} := [B_{1}]_{ji}\lambda_{1}\nabla f_{i}(x_{i}^{1})$ to infer $\nabla f_{i}(x_{i}^{1})$ since $x_{i}^{1}$, $\alpha_{i}$, and $[B_{1}]_{ji}$ are unknown. Finally, as elaborated in Section \ref{subsection_proposed_algorithm}, by combining (\ref{equation_system_equations_y}), attacker $j$ is also unable to infer $\nabla f_{i}(x_{i}^{K+1})$ due to $\lambda_{K+1} \to 0$. Therefore, attacker $j$ cannot steal the gradients of agent $i$.
\end{proof}

\section{Numerical Simulation}\label{section_simulation}
\begin{figure}[!t]
  \centering
  \includegraphics[width = 0.6\linewidth]{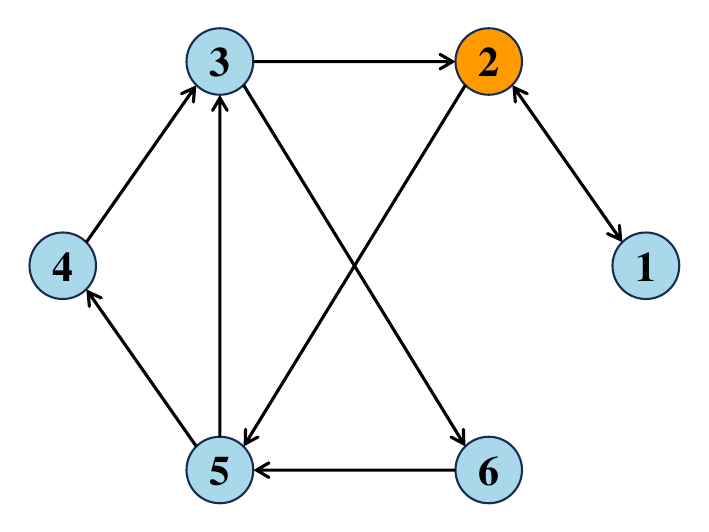}
  \caption{The topology of the communication network.}
  \label{simulation_network}
\end{figure}

In this section, we validate the effectiveness of the proposed algorithm through a canonical distributed estimation problem and the distributed training of a convolutional neural network (CNN). Consider a network system consisting of $n$ sensors. Each sensor has access to a measurement matrix $S_{i} \in \mathbb{R}^{d \times p}$ and collected measurement data $s_{i} \in \mathbb{R}^{d}$, satisfying $s_{i} = S_{i}x + \omega_{i}$, where $x \in \mathbb{R}^{p}$ is the unknown parameter to be estimated and $\omega_{i} \in \mathbb{R}^{d}$ represents Gaussian noise. All sensors collaborate to solve the following minimization problem to estimate $x$:
\begin{equation}\label{problem_sensor}
\min_{x\in\mathbb{R}^{p}} \sum_{i=1}^{n} \left( \| s_{i} - S_{i}x \|_{2}^{2} + r_{i}\| x \|_{2}^{2} \right),
\end{equation}
where $r_{i} \in \mathbb{R}$ is the regularization parameter.

In our simulations, we consider a directed network system with 6 sensors. The communication topology is illustrated in Fig. \ref{simulation_network}. Each sensor's measurement matrix $S_{i}$ is uniformly sampled from the interval $[0,10]$, with dimensions $d=3$ and $p=2$. The measurement data $s_{i}$ is calculated from predefined parameters $\tilde{x}$ uniformly sampled from $[0,1]$ and Gaussian noise $\omega_{i}$ with zero mean and unit standard deviation. The regularization parameters $r_{i}$ are uniformly set to 0.01.
\begin{figure}[!h]
  \centering
  \includegraphics[width = \linewidth]{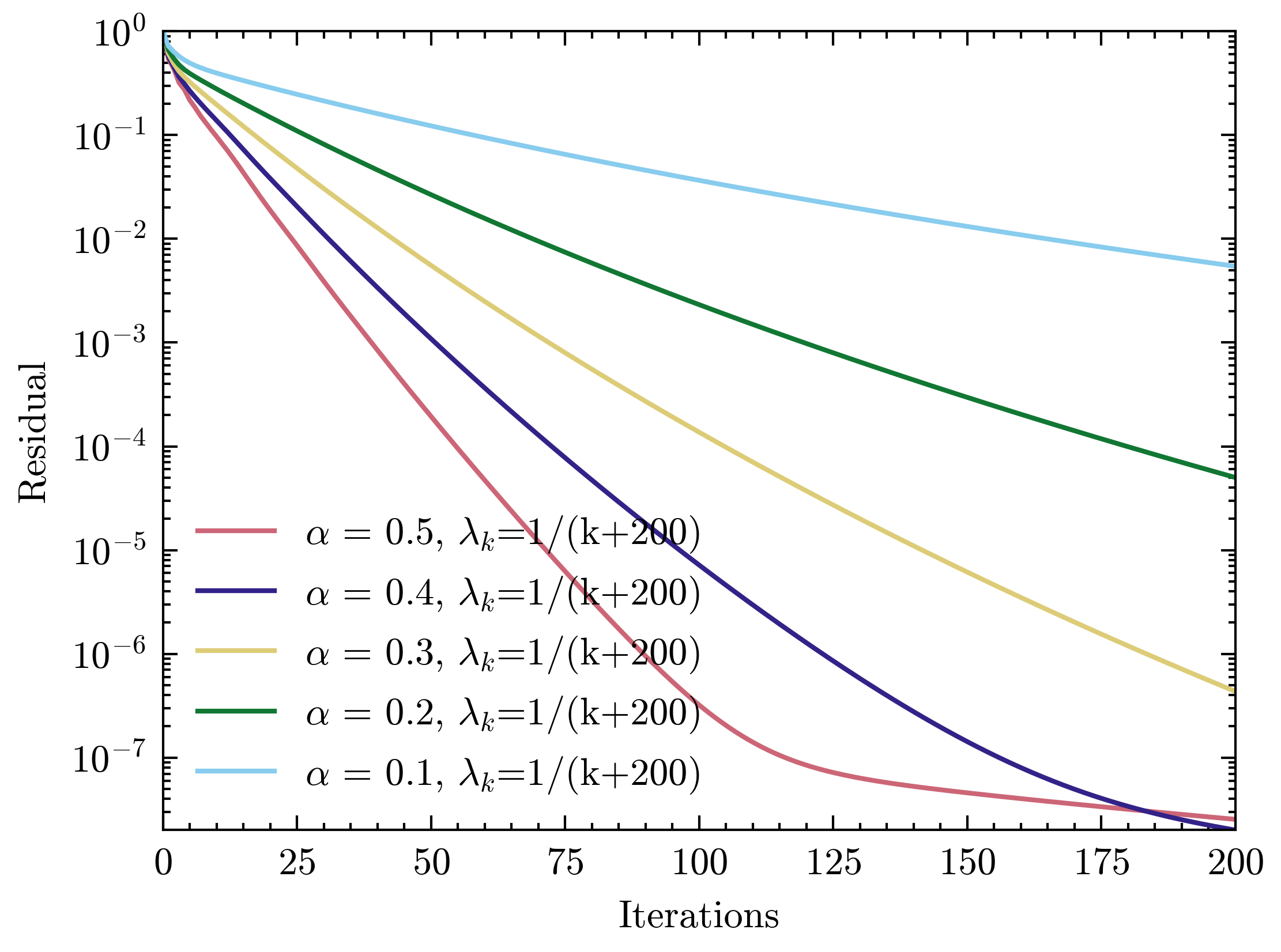}
  \caption{Convergence of the proposed algorithm under different parameters $\alpha$.}
  \label{fig_convergence_alpha}
\end{figure}

\begin{figure}[!h]
  \centering
  \includegraphics[width = \linewidth]{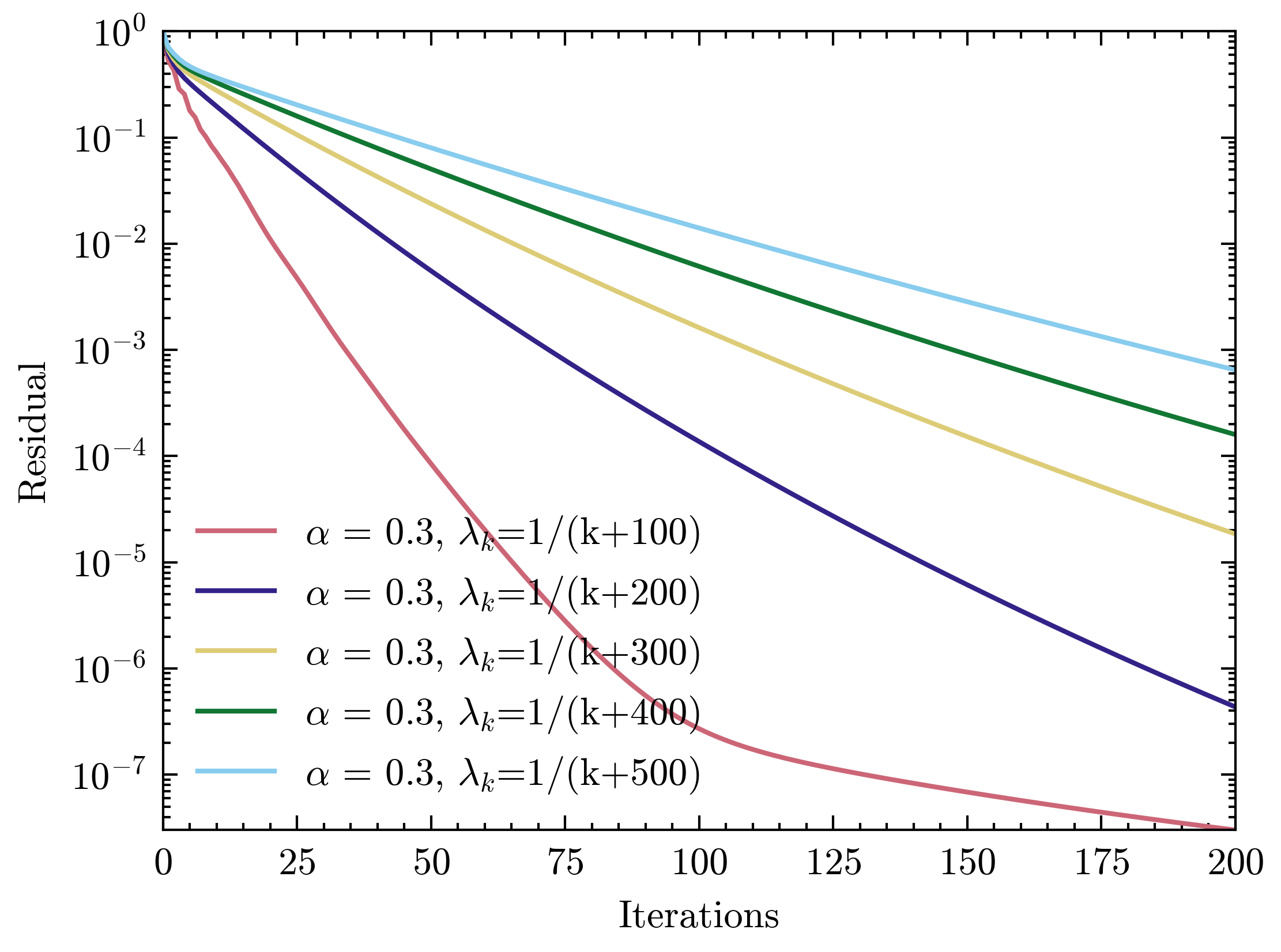}
  \caption{Convergence of the proposed algorithm under different parameters $m$.}
  \label{fig_convergence_m}
\end{figure}

\begin{figure}[!h]
  \centering
  \includegraphics[width = \linewidth]{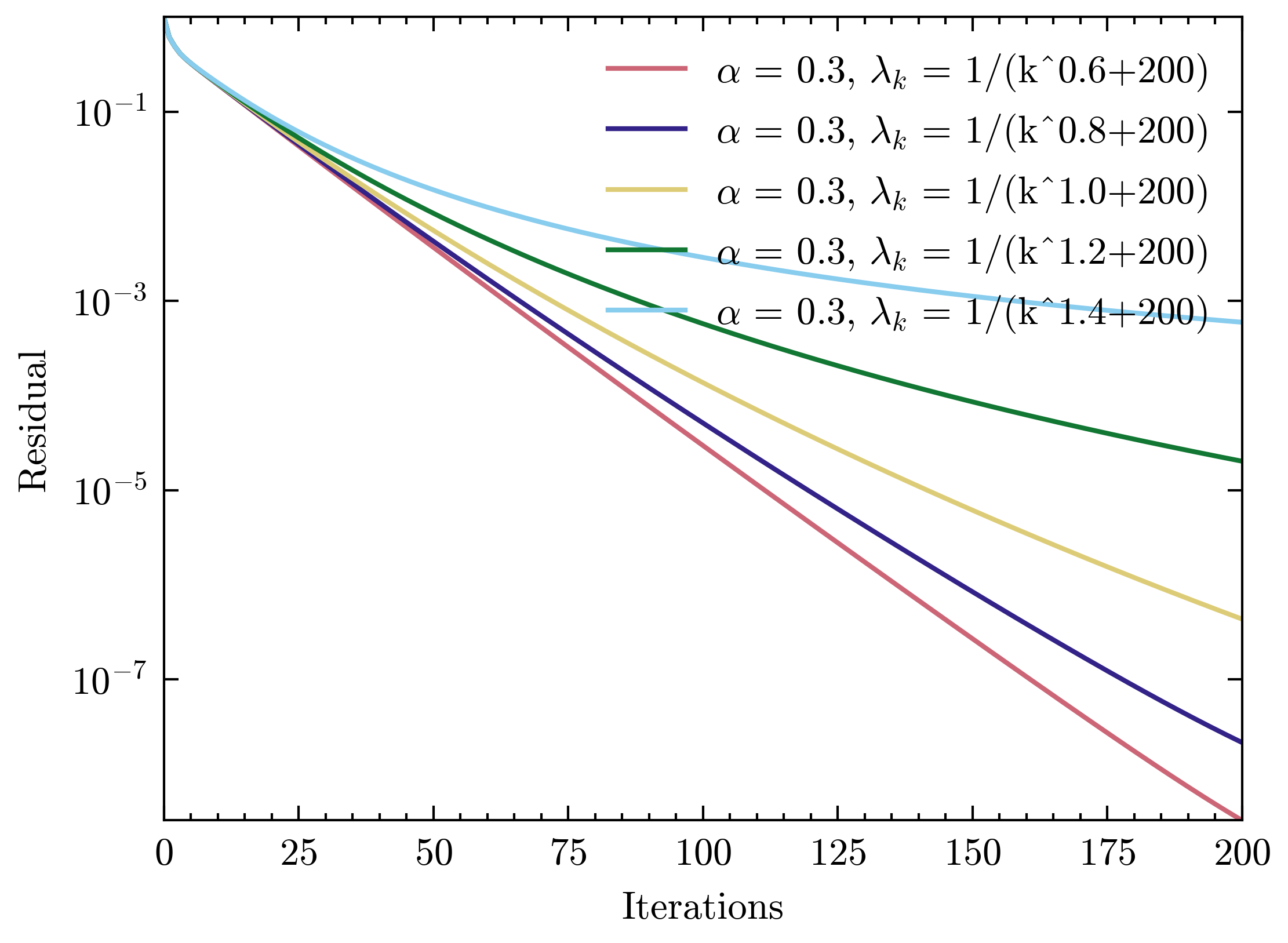}
  \caption{Convergence of the proposed algorithm under different parameters $e$.}
  \label{fig_convergence_e}
\end{figure}

\begin{table*}[!t]
\begin{center}
  \caption{The Inferred Results of DLG Attack for Different Algorithms.}
  \begin{tabular}{M{2cm}M{2cm}M{6.5cm}M{2cm}}
    \toprule
    Algorithm & Original Image & Process of DLG Attack & MSE \\
    \midrule
    AB \citep{saadatniaki2020decentralized} & \includegraphics[width=2cm]{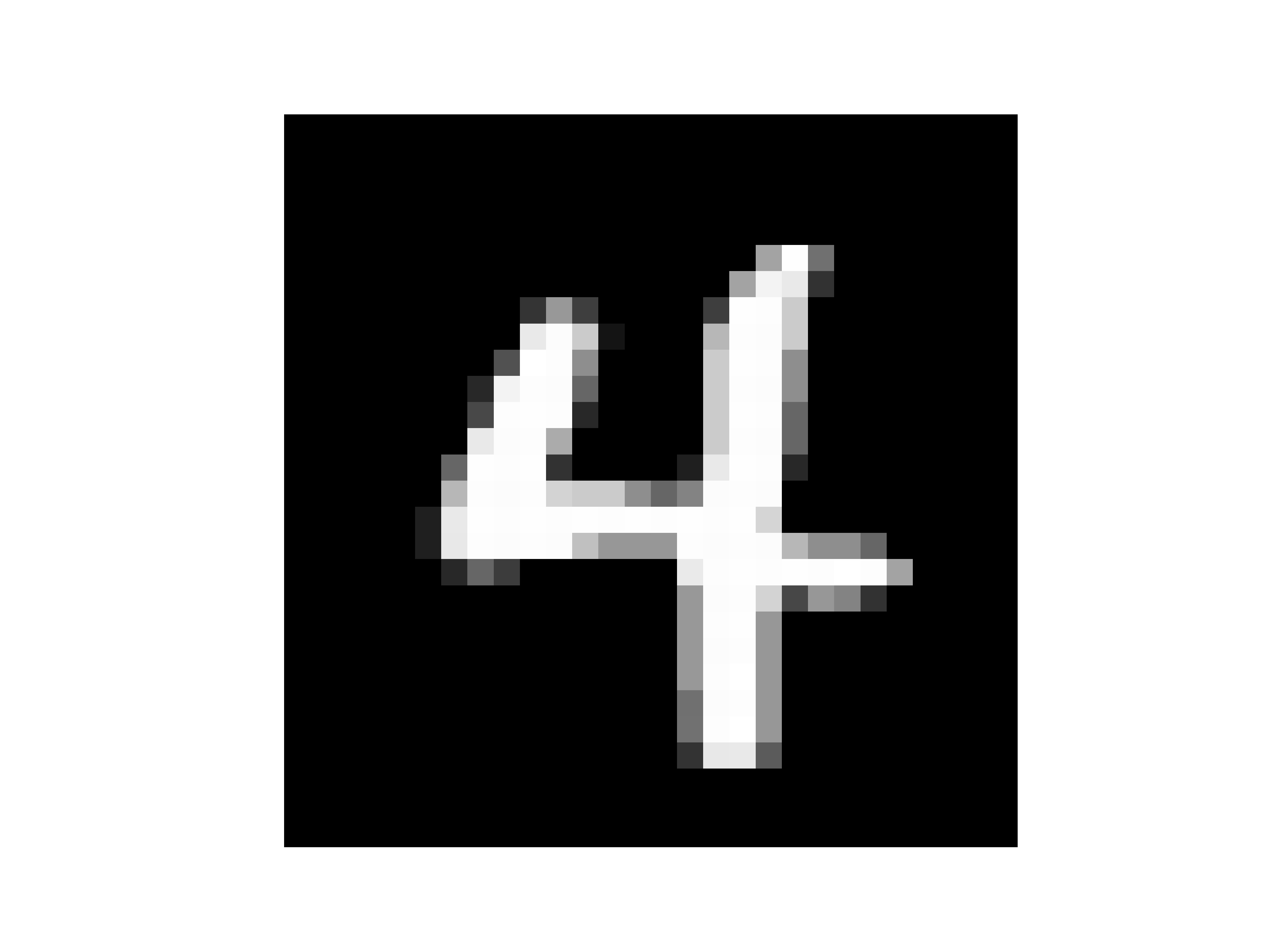} & \includegraphics[width=2cm]{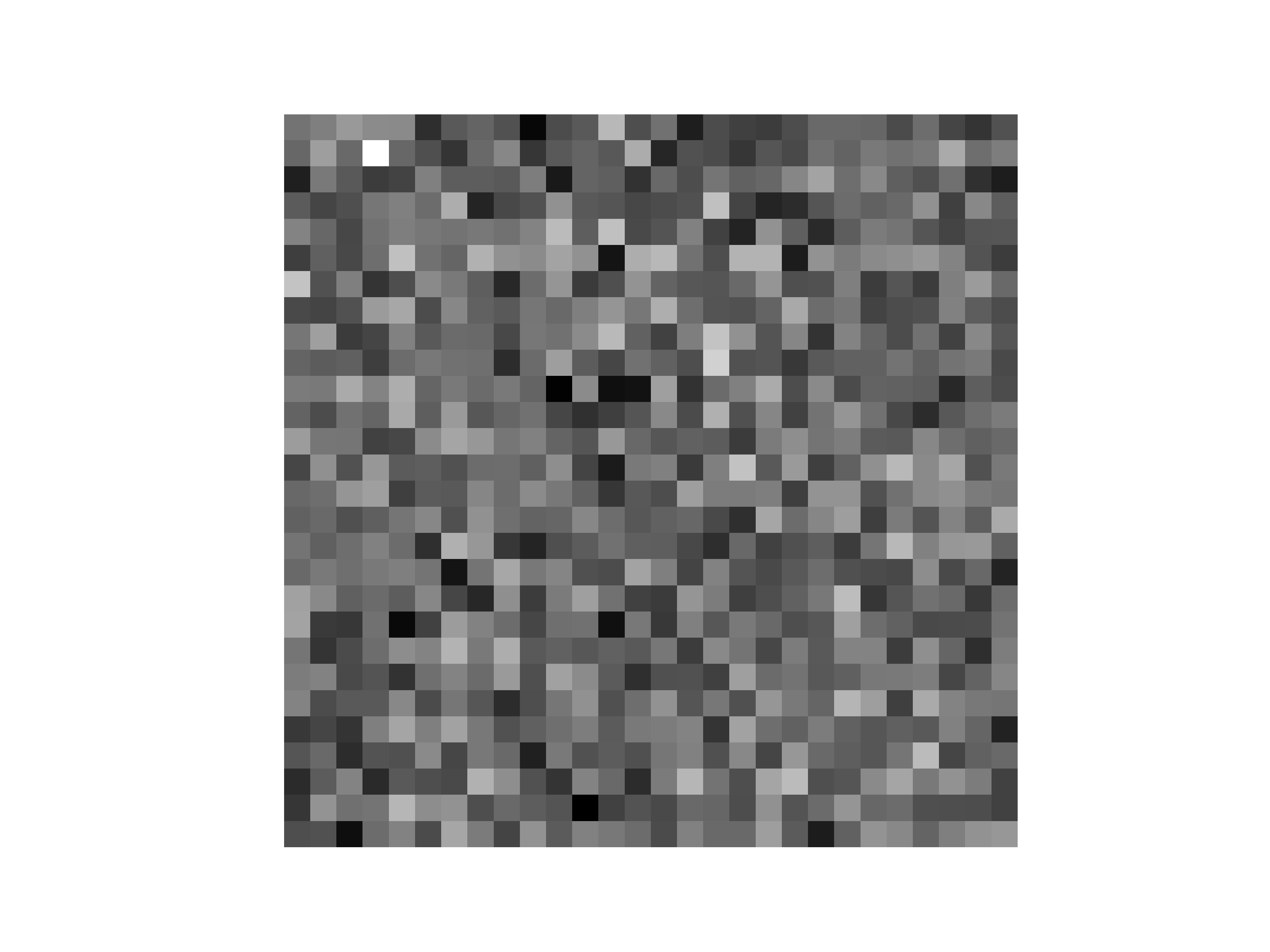} \includegraphics[width=2cm]{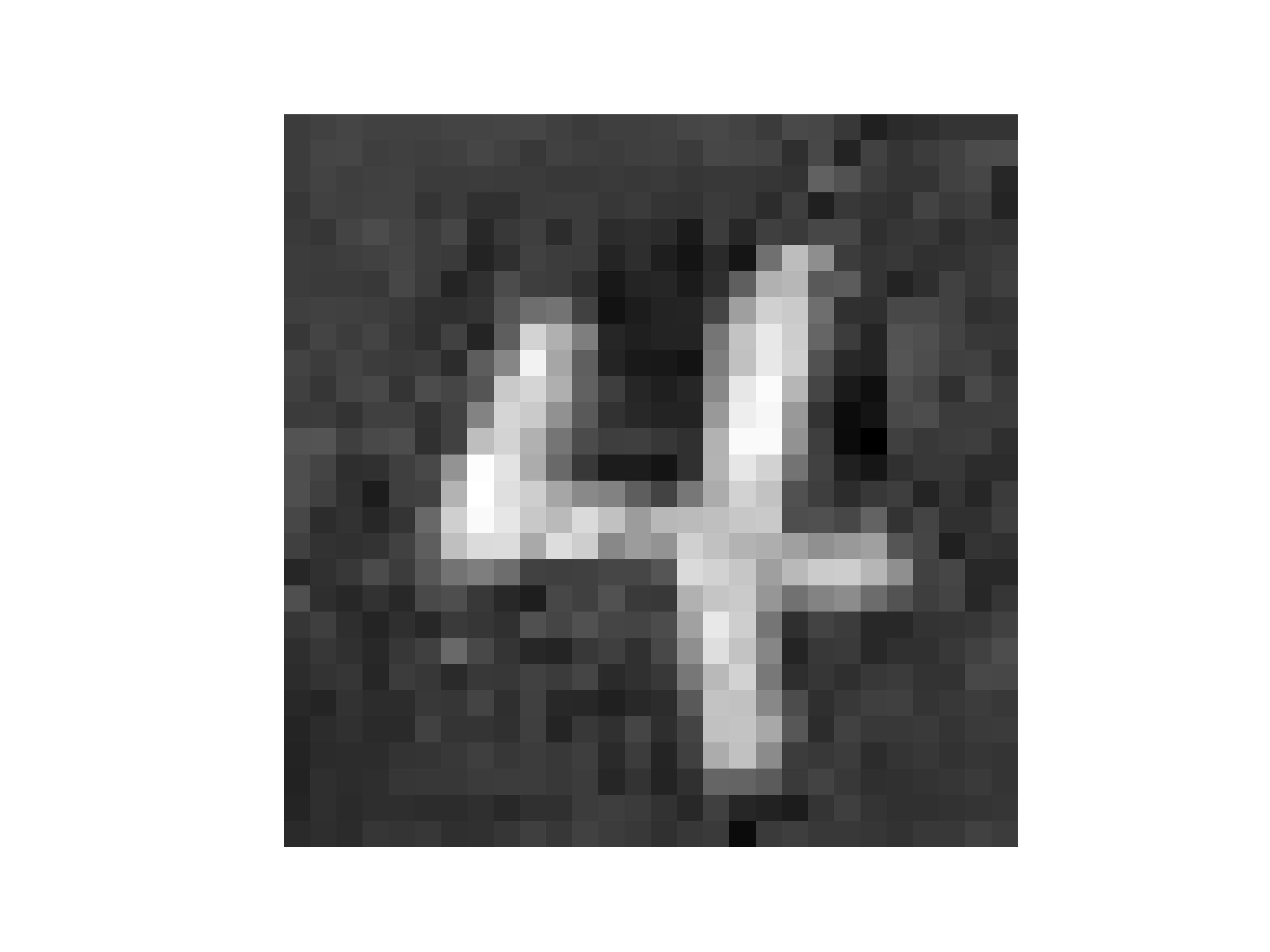} \includegraphics[width=2cm]{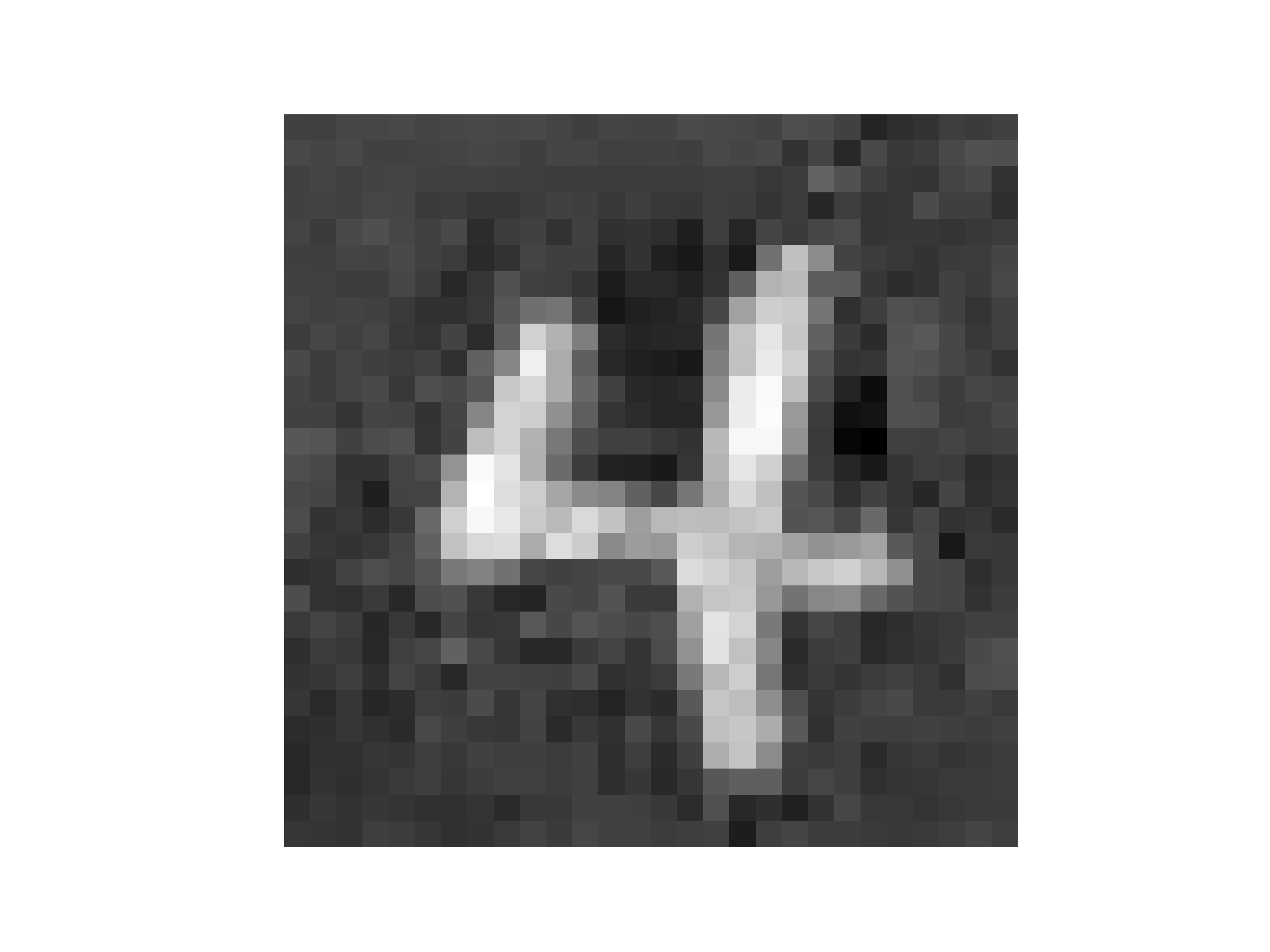} & $2.35\times 10^{-2}$ \\
    Our Method & \includegraphics[width=2cm]{original_image_0} & \includegraphics[width=2cm]{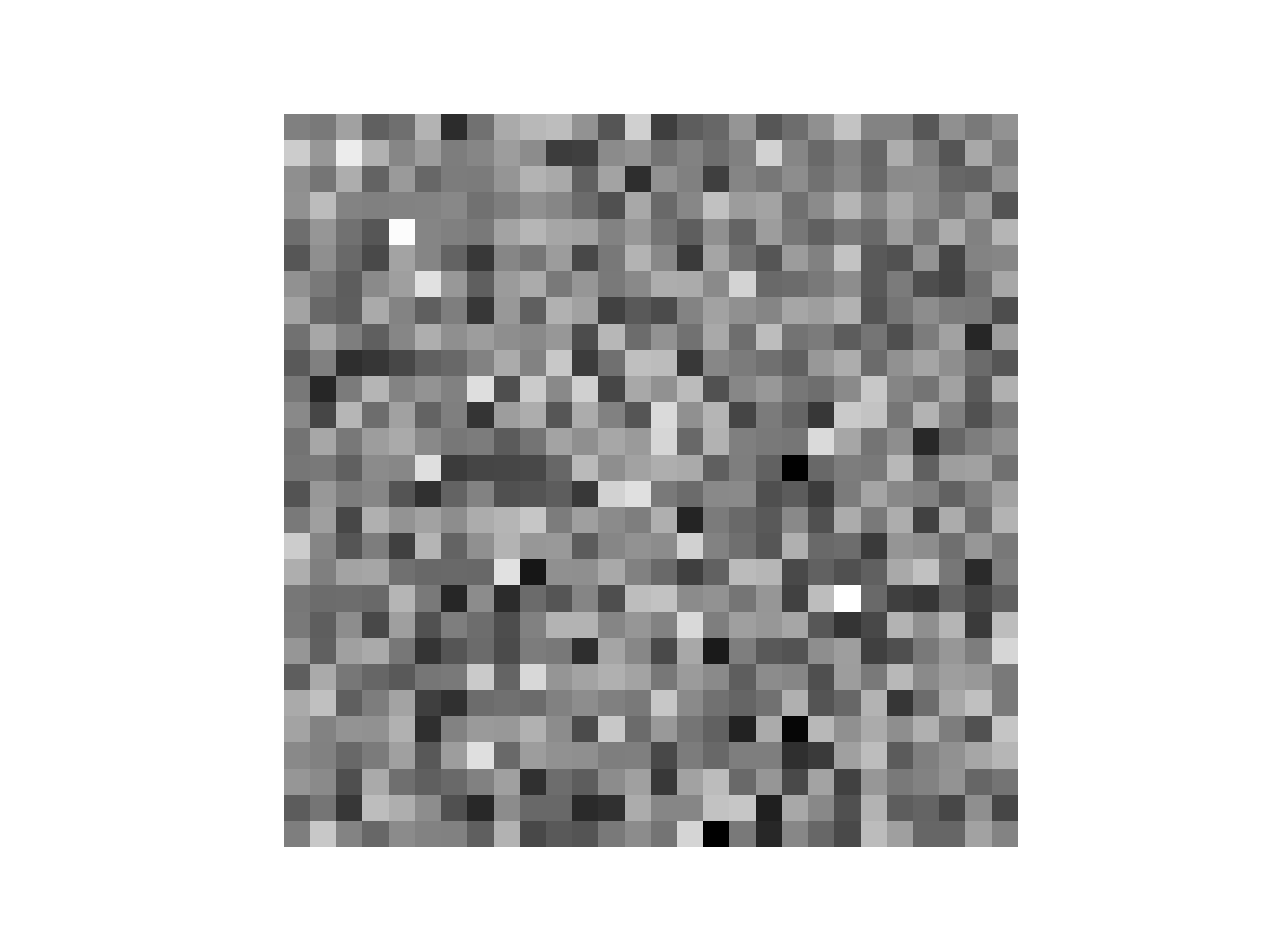} \includegraphics[width=2cm]{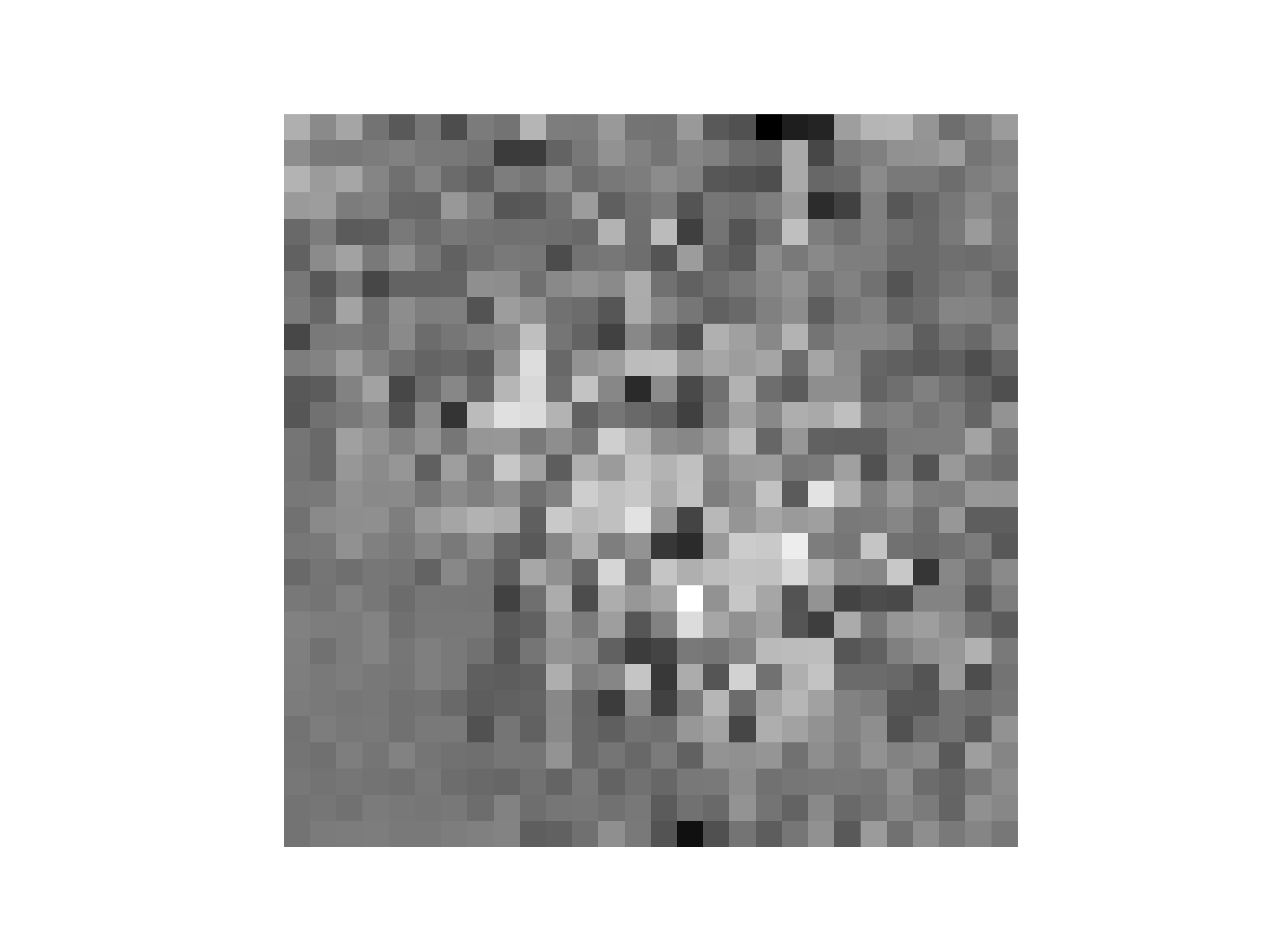} \includegraphics[width=2cm]{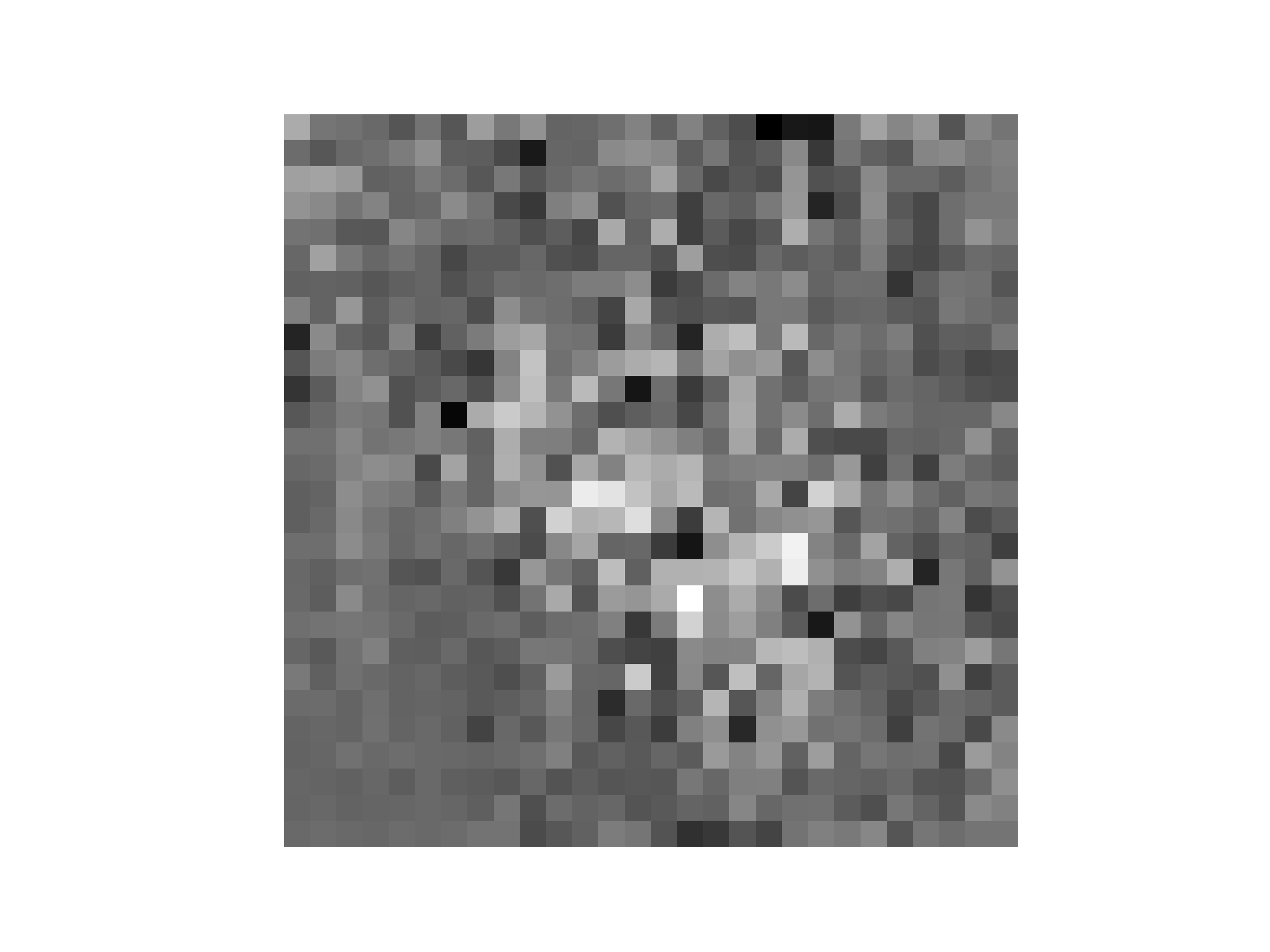} & 20.82 \\
    \bottomrule
  \end{tabular}
  \label{table_simulation_privacy}
\end{center}
\end{table*}

We verify the convergence of the proposed algorithm under different parameter settings. We set the step size of the agent to $\alpha_{i} = \alpha$ for all $i\in[n]$ and set the decaying weight factor to $\lambda_{k} = \frac{1}{k^{e}+m}$, where $e > 0$ and $m \ge 0$ are constants. Fig. \ref{fig_convergence_alpha}, \ref{fig_convergence_m}, and \ref{fig_convergence_e} illustrate the relative residuals of the algorithm (i.e., $\| \mathbf{x}^{k} - \mathbf{1}x^{*} \|_{2}^{2}/\| \mathbf{x}^{1} - \mathbf{1}x^{*} \|_{2}^{2}$) under different values of $\alpha$, $m$, and $e$, respectively. It can be observed that the algorithm rapidly achieves an acceptable level of optimization accuracy under the majority of parameter settings. Furthermore, we can deduce the following two rules: 1) the larger the step size $\alpha$, the faster the algorithm converges, and 2) the slower the weight factor $\lambda_{k}$ decays, the faster the algorithm converges. Both rules are consistent with the intuition of the algorithm.

We evaluate the privacy performance of the proposed algorithm through the distributed training of a CNN. We also use the communication network depicted in Fig. \ref{simulation_network}. All agents cooperatively train a CNN using the MNIST dataset of handwritten digits. Each agent has a local copy of the CNN and an image from the MNIST dataset. Agent 2 is an attacker attempting to recover agent 1's data using the DLG attack \citep{zhu2019deep}. We first conduct the experiment using the AB algorithm \citep{saadatniaki2020decentralized} with a step size of 0.01. In the AB algorithm, agent 2 can utilize equality (\ref{equation_leakage_GT}) to infer agent 1's gradient at iteration $K+1$ and then employ the DGL attack to recover agent 1's data. The experiment results are presented in the first row of Table \ref{table_simulation_privacy}, which reveals that the AB algorithm fails to protect agent 1's image data. This failure occurs because agent 2 can infer the exact gradient of agent 1 in the AB algorithm. The experimental results of our proposed algorithm are displayed in the second row of Table \ref{table_simulation_privacy}. The relevant parameters of the algorithm are set as $\alpha_{i} = 0.1, \lambda_{k} = \frac{1}{k^{0.8}+10}$. Agent 2 can use equality (\ref{equation_leakage_WGT}) to infer the gradient of agent 1 at iteration $K+1$. However, it can be observed that attacker 2 cannot successfully recover agent 1's original image, thereby confirming the privacy-preserving efficacy of the proposed algorithm.

\section{Conclusions}\label{section_conclusion}
In this paper, we propose a weighted gradient tracking-based distributed privacy-preserving algorithm, which eliminates the privacy leakage risk in gradient tracking by introducing a decaying weight factor. We demonstrate that the proposed algorithm can protect agents' private information, including their intermediate states and all gradients, from external eavesdroppers and honest-but-curious adversaries. The proposed privacy-preserving algorithm neither incurs additional computational and communication overheads nor requires extra topological assumptions. Furthermore, we characterize the convergence of the proposed algorithm under time-varying heterogeneous step sizes. Under the assumption that the objective function is strongly convex and smooth, we prove that the proposed algorithm converges precisely to the optimal solution. In future work, we will consider heterogeneous weight factors to enhance the algorithm's privacy further. Correspondingly, we need to develop a new analytical framework for the convergence analysis.




\bibliographystyle{myplainnat}        
\bibliography{mybibfile}           



\end{document}